\documentclass[10pt,twocolumn,twoside]{IEEEtran}
\usepackage{amsmath}%
\usepackage{amsthm}%
\usepackage{amsfonts}%
\usepackage{amssymb}%
\usepackage{graphicx}
\usepackage{algpseudocode}
\usepackage{algorithm}
\usepackage{booktabs} 
\usepackage{lscape}
\usepackage[space]{grffile}
\usepackage{epsfig,psfrag,amstext,subfigure,subeqnarray,nccfoots,color,multirow,bm}
\usepackage{subfigure} 
\usepackage{threeparttable}

\usepackage{amsmath,graphicx}
\usepackage{amsmath}
\usepackage{graphicx} 
\usepackage{amsthm}
\usepackage{algpseudocode}
\usepackage{algorithm}
\usepackage{epstopdf}
\usepackage{latexsym,bm,amsmath, amssymb,amsfonts}%
\usepackage{enumitem,enumerate}
\usepackage{booktabs}
\usepackage{multirow,url}

\usepackage{cuted}
\usepackage{color}
\usepackage{cite}
\usepackage{subfigure}
\usepackage{stfloats}
\newtheorem{fact}{Fact}
\newtheorem{thm}{Theorem}
\newtheorem{lem}{Lemma}
\newtheorem{prop}{Proposition}

\newtheorem{rmk}{Remark}

\theoremstyle{plain}

\definecolor{orange}{RGB}{255,107,0}

\ifCLASSINFOpdf
\else
\fi 
\begin{document}
	%
	\title{Stochastic Mirror Descent for Low-Rank Tensor Decomposition Under Non-Euclidean Losses}
	%
	%
	%
	%
	\author{Wenqiang Pu,
		Shahana Ibrahim,
		Xiao Fu, 
		and Mingyi Hong
		\thanks{X. Fu is supported in part by NSF ECCS 1808159,  IIS-1910118 and  ARO  award  W911NF-19-1-0247.   M. Hong  is  supported  in  part  by  NSF  Award  CIF-1910385,  and  ARO award W911NF-19-1-0247.}
		\thanks{Wenqiang Pu is with Shenzhen Research Institute of Big Data, China.}
		\thanks{Shahana Ibrahim and Xiao Fu are with Electrical Engineering and Computer Science, Oregon State University, Corvallis, USA.}
		\thanks{Mingyi Hong is with Electrical and Computer Engineering, University of Minnesota, MN, USA}}

	\maketitle
	
	\begin{abstract}
		This work considers low-rank canonical polyadic decomposition (CPD) under a class of non-Euclidean loss functions that frequently arise in statistical machine learning and signal processing. These loss functions are often used for certain types of tensor data, e.g., count and binary tensors, where the least squares loss is considered unnatural.
		Compared to the least squares loss, the non-Euclidean losses are generally more challenging to handle. 
		Non-Euclidean CPD has attracted considerable interests and a number of prior works exist.
		However, pressing computational and theoretical challenges, such as scalability and convergence issues, still remain.
		This work offers a unified stochastic algorithmic framework for large-scale CPD decomposition under a variety of non-Euclidean loss functions. Our key contribution lies in a tensor fiber sampling strategy-based flexible stochastic mirror descent framework. Leveraging the sampling scheme and the multilinear algebraic structure of low-rank tensors, the proposed lightweight algorithm ensures global convergence to a stationary point under reasonable conditions. Numerical results show that our framework attains promising non-Euclidean CPD performance. The proposed framework also exhibits substantial computational savings compared to state-of-the-art methods.
	\end{abstract}
	

	%
	\IEEEpeerreviewmaketitle

	\section{Introduction}
	\label{sec:intro}
	
	\textit{Canonical polyadic decomposition} (CPD) has been used in many core tasks in signal processing and machine learning, such as neural signal analysis, video processing, array signal processing, text mining, social network analysis, link prediction, among others---see \cite{sidiropoulos2017tensor,kolda2009tensor,Cichocki2015,anandkumar2014tensor}.
	

	The majority of classic CPD models and algorithms were developed for least squares (LS) problems using the Euclidean distance-based fitting criterion; see \cite{kolda2009tensor,Cichocki2015,sidiropoulos2017tensor,fu2020nonconvex} and references therein. However, the Euclidean distance is unnatural for measuring the ``distance'' between many types of real-world data, e.g., stochastic, integer, and binary data. In principle, using certain ``data geometry-aware'' divergences to serve as the fitting criteria may greatly improve performance and robustness in practice~\cite{hong2020generalized,wang2020learning,vandecappelle2020second}. For example, the ``distance'' between two probability distributions is typically measured by a proper divergence, such as the generalized {\it Kullback-Leibler} (KL) divergence~\cite{huang2017kullback,kargas2019learning,Fu2019Link,cheng2020learning} and Itakura-Saito (IS) divergence~\cite{ermics2015link}. These divergences take into consideration that the data is constrained in the probabilistic simplex, and thus are often more effective relative to the LS criterion in analyzing data that are not generated over the entire Euclidean space.
	From a statistical estimation viewpoint, many non-Euclidean divergences are closely related to the {\it maximum likelihood estimators} (MLEs) under plausible data distributions.
	For example, the generalized KL divergence~\cite{chi2012tensors} and logistic loss~\cite{hong2020generalized,wang2020learning} can be derived from the MLEs of count integer data and binary data that follow certain Poisson distributions and Bernoulli distributions, respectively.

	However, computing CPD under non-Euclidean divergences is much more challenging compared with the case under Euclidean loss (or, the LS loss), especially when the data size becomes huge. 
	Algorithms developed under the LS loss are often not easily extendable to handle these more complicated loss functions, due to the lack of ``nice'' properties that are possessed by the LS loss, e.g., the gradient Lipschitz continuity under relatively mild conditions.
	Below, we provide a brief review on existing developments for CPD models with specific loss function. 

	\subsection{Prior Works}
	\newcommand{\minitab}[2][l]{\begin{tabular}{#1}#2\end{tabular}}
	\begin{table*}
		\centering
		\caption{Brief review of algorithms for CPD model.}
		\small
		\begin{threeparttable}
			\begin{tabular}{c c c c c}
				\toprule
				\textbf{Algorithm}  & \textbf{Stochastic}  & \textbf{Loss function} & \textbf{Data Type} \\
				\midrule
				First-order type algorithm~\cite{comon2009tensor,phan2013fast,xu2013block}& No &  LS & continuous \\
				Primal-dual algorithm~\cite{huang2016flexible}& No &  LS & continuous \\
				(Quasi-)Second-order algorithm~\cite{sorber2013optimization,phan2013low}& No &  LS & continuous \\
				Stochastic optimization algorithm~\cite{beutel2014flexifact,vervliet2015randomized,Sorensen2019,fu2020block}& Yes &  LS & continuous \\
				Hierarchical alternating optimization~\cite{cichocki2009fast}& No &  $\alpha,\beta-$ Div. & continuous \\
				Majorization-minimization algorithm~\cite{chi2012tensors}& No &  KL Div. & continuous and count \\
				Multiplicative update algorithm~\cite{ermics2015link}& No &  LS, KL Div., IS Div. & binary \\
				Exponential gradient algorithm~\cite{kargas2019learning}& No &  KL Div. & continuous and count \\
				Alternating optimization algorithm~\cite{wang2020learning}& No &  logistic loss and others & binary \\
				Generalized Gaussian Newton algorithm~\cite{vandecappelle2020second}& No &  $\beta$-div. & continuous and count \\
				Stochastic gradient descent algorithm~\cite{hong2020generalized}& Yes &  general loss$^*$ & continuous, count, and binary \\
				\textbf{Stochastic mirror descent algorithm (this work)}& Yes &  general loss$^*$& continuous, count, and binary\\
				\bottomrule
			\end{tabular}
			\begin{tablenotes}
				\footnotesize
				\item {\scriptsize *The general loss in this table refers to many ML criterion motivated losses~\cite{hong2020generalized} as well as statistical divergences such as KL div., IS div and etc. }
			\end{tablenotes}
		\end{threeparttable}
		\label{tab:alg}
	\end{table*}
	Many existing non-Euclidean CPD approaches employ the {\it block coordinate descent} (BCD) paradigm~\cite{razaviyayn2013unified} with divergence-specific strategies for block variable updating.
	For example, the work in \cite{cichocki2009fast} proposed a hierarchical alternating optimization algorithm for CPD with $\alpha$- and $\beta$-divergence. In~\cite{chi2012tensors}, the generalized KL-divergence loss was considered, where a block majorization-minimization (MM) algorithm was developed. In~\cite{kargas2019learning}, the exponential gradient algorithm was proposed for the KL-divergence. Similar strategies were developed for the the KL and IS divergences \cite{ermics2015link}. 
	Recently in~\cite{wang2020learning}, several ML-based loss functions for binary data were considered and an alternating optimization algorithm with line search was proposed. 
	Besides BCD, other optimization frameworks such as Gauss-Newton based methods~\cite{vandecappelle2020second} and stochastic gradient-based methods~\cite{hong2020generalized} were also developed for non-Euclidean CPD.
	
	
	It is important to note that most of the algorithms mentioned above (such as \cite{cichocki2009fast,chi2012tensors,fevotte2011algorithms,ermics2015link,wang2020learning,vandecappelle2020second}),  are {\it batch} algorithms, which utilize the entire data set to perform every update.  
	They become increasingly slow when the size of the data increases.  
	On the other hand, {\it stochastic} algorithms are effective in reducing per-iteration computational and memory burdens. Recently, a stochastic gradient descent (SGD) based algorithm~\cite{hong2020generalized} was proposed for CPD with statistical criterion based loss functions. The algorithm was developed based on randomly sampling the tensor entries. Hence, it is difficult to exploit some interesting multilinear algebraic properties of low-rank tensors to further improve computational efficiency. 
	In addition, the SGD algorithm in~\cite{hong2020generalized} lacks convergence guarantees. In Table~\ref{tab:alg}, we  summarize the properties of a number of recently developed algorithms for the CPD model.

	\subsection{Contributions}
	In this paper, we develop a \textit{unified}  {\it stochastic mirror descent} (SMD) algorithmic framework for large-scale CPD under various non-Euclidean losses. Our major contributions are summarized as follows:

	\noindent
	$\bullet$ \textbf{Efficient fiber-sampled stochastic MD framework:} 
	We propose a block-randomized SMD algorithmic framework that is tailored for tensor decomposition. 
	Both MD and SMD are known for its effectiveness in handling non-Euclidean losses~\cite{beck2003mirror}, but directly applying generic SMD does not fully exploit the underlying CPD structure.
	We use a recently emerged tensor data sampling strategy (namely, fiber sampling \cite{battaglino2018practical,fu2020block}) to assist designing SMD-type updates. The fiber sampling strategy judiciously uses the multilinear structure of low-rank tensors, which gives rise to structured (non-)convex subproblems. These structures can often be exploited to come up with economical update rules for CPD.

	\noindent
	$\bullet$ \textbf{A suite of solutions for various losses and constraints:} We carefully craft solutions for a series of non-Euclidean losses.
	The proposed algorithmic framework allows flexible choices of the local surrogate functions under the SMD framework to adapt to different loss functions.~Such flexibility also helps offer lightweight updates when the latent factors are under a variety of constraints that are of interest in data analytics. In particular, we pay special attention to binary and integer data CPD problems, which find numerous applications across disciplines.


	\noindent
	$\bullet$ \textbf{Guaranteed convergence:} 
	We offer convergence characterizations for our block-randomized SMD-based non-Euclidean CPD framework. Establishing stationary-point convergence for generic SMD is already a challenging problem.
	The work in~\cite{dang2015stochastic} on SMD requires its gradient estimation error converging to zero, which is unrealistic in many cases, especially under the context of CPD. In this work, we leverage the notion of {\it relative smoothness} and the tensor fiber sampling strategy to construct lightweight SMD updates for different losses. This design also helps circumvent stringent conditions (e.g., vanishing gradient estimation error) when establishing convergence. To our best knowledge, such convergence results for multi-block SMD under nonconvex settings have been elusive in the literature.
	
	Part of the work will appear in IEEE ICASSP 2021. The conference version considered algorithm design under the $\beta$-divergence loss. 
	This journal version extends the ideas to handle more non-Euclidean losses, e.g., the logistic loss that is critical in binary data analysis. More importantly, this version provides unified convergence analysis for the proposed algorithmic structure. Some important practical considerations, e.g., stepsize scheduling, is also discussed and experimented. The journal also contains substantially more simulations and real-data validation.

	\subsection{Notation}
	We follow the conventional notation in signal processing, $x$, $\bm{x}$, and $\bm{X}$, and $\underline{\bm{X}}$ denote scalar, vector, matrix, and tensor, respectively. Given a matrix $\bm{X}$, $\bm{X}^{.c}$ and $\exp(\bm{X})$ denote {the} entry-wise power and exponential operations respectively; $\mathrm{vec}(\bm{X})$ denote the vectorization operator that concatenates the columns of $\bm{X}$. We use $\circledast$, $\odot$, and $\oslash$ to denote the Hadamard product, the Khatri-Rao product, and entry-wise division respectively. $^T$ denotes the transpose operation. Script letter $\mathcal{C}$ is used to denote a discrete set and $|\mathcal{C}|$ is the cardinality $\mathcal{C}$. $\| \cdot \|$ denotes the Euclidean norm of vector, $\|\cdot \|_F$ denotes the Frobenius norm of matrix, and $\langle \bm{x}, \bm{y} \rangle$ denotes the inner production of vectors $\bm{x}$ and $\bm{y}$. Other notation will be explained when it first appears.

	\section{CPD under Non-Euclidean Losses}
	Consider a data tensor $\underline{\bm{X}}\in\mathbb{R}^{I_1\times I_2\times \ldots \times I_N}$, where $I_n>0,n\in[N]$ is the size of the $n$th mode of $\underline{\bm{X}}$. Such multi-way data tensors arise in many applications. The entries of the data tensor $\underline{\bm{X}}$ could be continuous real numbers, 
	non-negative integers or binaries. 
	A general problem of interest is to approximate $\underline{\bm{X}}$ using a low rank tensor $\underline{\bm{M}}$, defined as 
	\begin{equation}\label{eq:CPDmodel}
		\underline{\bm{M}}=\sum\nolimits_{r=1}^R \bm{A}_1(:,r)\circ \bm{A}_2(:,r) \circ \ldots \circ  \bm{A}_N(:,r),
	\end{equation}
	where ``$\circ$'' denotes the outer product of vectors, $\bm{A}_n\in\mathbb{R}^{I_n\times R}$ is the mode-$n$ latent factor matrix; $R$ is the smallest positive integer such that \eqref{eq:CPDmodel} holds, and it is also known as the rank of $\underline{\bm{M}}$.
	
	Denote an $N$-dimensional integer vector $\bm{i}$ as the entry coordinate, i.e., $$\bm{i}\in\mathcal{I}\triangleq\{ (i_1,i_2,\ldots, i_N)|\  i_n=1,2,\ldots,I_n,\forall n\}.$$
	Then the CPD problem can be formulated as the following minimization problem with a loss function of interest $\ell(\cdot,\cdot):\mathbb{R}\times\mathbb{R}\mapsto\mathbb{R}$,
	\begin{equation}\label{eq:optori}
		\begin{aligned}
			\min_{\bm{A}_1,\bm{A}_2,\ldots,\bm{A}_N}\ &\frac{1}{|\mathcal{I}|}\sum_{\bm{i}\in\mathcal{I}}  \ell \left({\underline{\bm{X}}_{\bm{i}}},\underline{\bm{M}}_{\bm{i}}\right)\\
			\textrm{s.t.}\quad\quad& {\underline{\bm{M}}_{\bm{i}}}={\sum\nolimits_{r=1}^R \prod\nolimits_{n=1}^N} \bm{A}_n(i_n,r),\ \forall \bm{i}\in\mathcal{I},\\
			&\bm{A}_n\in\mathcal{\bm{A}}_n,\ \forall n,
		\end{aligned}
	\end{equation}
	where ${\underline{\bm{X}}_{\bm{i}}}$ and ${\underline{\bm{M}}_{\bm{i}}}$ denote the entries of ${\underline{\bm{X}}}$ and ${\underline{\bm{M}}}$ indexed by $\bm{i}$, respectively, $\mathcal{\bm{A}}_n$ is a {constraint} set which captures the prior information about the structure of latent factors $\bm{A}_n$, e.g., non-negativity, sparsity, and smoothness. {By choosing proper loss functions $\ell$, Problem~\eqref{eq:optori} is used for handling different types of data, e.g., continuous, count, and binary data. Several representative motivating examples are as follows:}
	
	\begin{table*}
		\centering
		\caption{Distributions, link functions, and loss functions for different types of data.}
		\small
		\begin{threeparttable}
			\begin{tabular}{c c c c c c}
				\toprule
				\textbf{Data Type}  & \textbf{Distribution}  & \textbf{Link Function}  &\textbf{Loss function} & \textbf{Parameter Type} & \textbf{Name} \\
				\midrule
				& Gaussian &  $\theta(m)=m$ &$\frac{1}{2}(x-m)^2$ & $x,m\in\mathbb{R}$ & Euclidean Dis. \\
				Continuous & Gamma &  $\theta(m)=m$ &  $\frac{x}{m+\epsilon}+\log (m+\epsilon)$ & $x>0,m\geq0$  & IS Div. \\
				& -- & -- & $(m+\epsilon)^\beta/\beta-x (m+\epsilon)^{\beta-1}/(\beta-1)$ & $x\geq 0,m\geq0$ & $\beta$-Div. $\beta\in\mathbb{R}/\{0,1,2\}$\\
				\midrule
				Count & Poisson & $\theta(m)=m$ &  $m-x\log(m+\epsilon)$ & $x\in\mathbb{N},m\geq 0$ & generalized KL Div., \\
				&   &$\theta(m)=e^m$&  $e^m-xm$ & $x\in\mathbb{N},m\in\mathbb{R}$ & -- \\
				\midrule
				Binary &Bernoulli &  $\theta(m)=\frac{m}{1+m}$ &$\log(m+1)-x\log(m+\epsilon)$ & $x\in\{0,1\},m\geq0$ & -- \\
				&  &  $\theta(m)=\frac{e^m}{1+e^m}$& $\log(1+e^m)-xm$ & $x\in\{0,1\},m\in\mathbb{R}$ & --\\
				
				\bottomrule
			\end{tabular}
		\end{threeparttable}
		\label{tab:ell}
	\end{table*}
	

	{
		\noindent $\bullet$ {\bf KL-divergence for count data:} In many real-world scenarios, data is naturally recorded as nonnegative integers, {e.g.,
			crime numbers across locations and time\footnote{\scriptsize See official website of the city of Chicago, \texttt{www.cityofchicago.org.}} and email interactions recorded over months~\cite{shetty2004enron}.
		} 
		{As an information-theoretic measure, the KL divergence was originally proposed for quantifying similarity between two probability distributions.}
		The {\it generalized} KL-divergence that handles nonngeative quantities beyond distributions is also widely used in data analytics~\cite{chi2012tensors,Fu2019Link,huang2017kullback,fevotte2009nonnegative}. The generalized KL divergence has an $\ell$ defined as follows:
		\begin{equation}\label{eq:genKL}
			(\mathtt{KL}\text{-}\mathtt{Div.})\quad\ell(x,m)=m-x\log(m),
		\end{equation}
		where $ \ x\in\mathbb{N}$ and $m\geq0$. 
		Problem~\eqref{eq:optori} with KL-divergence can also be interpreted as the MLE for  estimating the Poisson parameter tensor $\underline{\bm{M}}$, which has a low-rank structure~\cite{chi2012tensors,hong2020generalized}. The corresponding statistical model is  $$\underline{\bm{X}}_{\bm{i}}\sim \mathtt{Poisson}(\underline{\bm{M}}_{\bm{i}}),\forall \bm{i},$$
		where $\mathtt{Poisson}(m)$ denotes the Poisson distribution with a mean of $m$. 
		
		
		\noindent $\bullet$ {\bf Log loss for binary data:} Binary data  is also frequently encountered in data analytics, e.g., in adjacency matrix-based social network community detection~\cite{huang2019detecting,anandkumar2014tensor} and knowledge base analysis~\cite{hong2020generalized,wang2020learning}. Binary data fitting is often handled using the following loss:
		\begin{equation}\label{eq:genLog}
			(\mathtt{Log\ Loss})\quad\ell(x,m)=\log(1+e^m)-xm,
		\end{equation}
		where $x\in\{0,1\}$ and $m\geq 0$. The log loss can be interpreted as MLE for finding the Bernoulli distribution parameter~\cite{wang2020learning}. The associated binary data generation model is
		$$\underline{\bm{X}}_{\bm{i}}\sim\mathtt{Bernoulli}(\underline{\bm{\Theta}}_{\bm{i}}),\ \underline{\bm{\Theta}}_{\bm{i}}=\underline{\bm{M}}_{\bm{i}}/(1+\underline{\bm{M}}_{\bm{i}}),\forall \bm{i},$$
		where $\mathtt{Bernoulli}(\theta)$ denotes the Bernoulli distribution, $\theta$ is the probability for $x$ taking $1$, and $\underline{\bm{\Theta}}$ has the same size of $\underline{\bm{X}}$. 
		
		
		\noindent $\bullet$ {\bf $\beta$-divergence:} 
		Non-Euclidean losses also find applications in continuous data CPD, especially under non-Gaussian and/or non-additive noise, e.g., multiplicative Gamma noise~\cite{vandecappelle2020second}. For example, the $\beta$-divergence was found useful for neural signal analysis~\cite{cichocki2009fast} and recently is studied as CPD fitting criterion~\cite{hong2020generalized,vandecappelle2020second,pu2021fiber}. The $\beta$-divergence is parametrized by a constant $\beta\in\mathbb{R}$ defined as 
		\begin{equation*}
			(\beta\text{-}\mathtt{Div.})\quad\ell(x,m)=
			\begin{cases}
				\frac{x}{m}-\log(\frac{x}{m})-1,&\beta=0,\\
				x\log\frac{x}{m}+m-x,&\beta=1,\\
				\frac{\left(x^\beta+(\beta-1)m^\beta-\beta x y^{\beta-1}\right)}{\beta(\beta-1)},& \textrm{o.w.}
			\end{cases}
		\end{equation*}
		The $\beta$-divergence subsumes the IS divergence ($\beta=0$), the generalized KL divergence ($\beta=1$), and the Euclidean distance ($\beta=2$) as special cases. When $\beta=0$, it can also be interpreted as MLE corresponds to data with multiplicative Gamma noise. In music data analysis, $\beta<2$ was found useful, since such loss functions capture low intensity spectra components---but the Euclidean loss tends to focus on significant variations in data~\cite{fevotte2009nonnegative}. 
		
		%
		%
		\begin{rmk}
			As one has seen in the examples,
			one way to select $\ell$ is to take a statistical analysis viewpoint. Each entry of the data tensor is treated as a random variable (r.v.) that is generated as follows:
			\begin{equation}\label{eq:x_modelM}
				\underline{\bm{X}}_{\bm{i}}\sim p\left(\underline{\bm{X}}_{\bm{i}}\mid \theta(\underline{\bm{M}}_{\bm{i}})\right),\ \forall \bm{i}\in\mathcal{I}.
			\end{equation}
			where $p(x;\theta)$ is a distribution with natural parameter $\theta$ (e.g., the Poisson and Bernoulli distribution) and $\theta(\cdot):\mathbb{R}\mapsto\mathbb{R}$ is an invertible function whose inverse is often referred as the \textit{link function} in statistics (e.g., $\theta(m)=m$ and $\theta(m)=\frac{m}{1+m}$). A straightforward intuition behind model~\eqref{eq:x_modelM} is that, the observed data tensor $\underline{\bm{X}}$ is `embedded' on a \textit{latent} low rank tensor $\underline{\bm{M}}$, whose generation procedure is characterized by $p(x;\theta)$ and $\theta(\cdot)$. To find $\underline{\bm{M}}$, a statistically efficient way is choosing $\ell(\cdot)$ as the negative log-likelihood function associated with model~\eqref{eq:x_modelM}, given as
			\begin{equation*}
				\ell(\underline{\bm{X}}_{\bm{i}},\underline{\bm{M}}_{\bm{i}})\triangleq-\log p(\underline{\bm{X}}_{\bm{i}}\mid \theta(\underline{\bm{M}}_{\bm{i}}))+\textrm{constant},
			\end{equation*}
			which naturally leads to a non-Euclidean CPD problem (if the distribution is not Gaussian).
			
			In Table~\ref{tab:ell}, some frequently used $\ell(\cdot)$, link function $\theta(\cdot)$, and distribution of our interests are illustrated. {In the table, $\epsilon>0$ is a sufficiently small number to avoid the function value or gradient being $\pm\infty$, i.e., $\epsilon=10^{-9}$. This modification is often used in the literature~\cite{hong2020generalized}.} 
		\end{rmk}
	}

	{
		Tackling Problem~\eqref{eq:optori} at scale is highly nontrivial.  
		For example, a $5000\times 5000\times 5000$ tensor can be as large as $900$GB if the double precision arithmetic is used,
		which means that batch algorithms may not be a viable option.
		Instead, stochastic algorithms that sample `partial data' per iteration become an attractive choice. In Euclidean loss CPD, it has been observed that stochastic algorithms can significantly reduce computational and memory cost per iteration; see \cite{vervliet2015randomized,battaglino2018practical,fu2020block}. {Nonetheless, unlike Euclidean CPD, various data sampling schemes and update rules may all offer competitive algorithms~\cite{fu2020nonconvex}, non-Euclidean losses' complex structures may make stochastic algorithm design a more delicate process. In other words, the sampling schemes may affect the subsequent update rules' complexity and convergence properties of the algorithm.}
		
		Next, we offer a unified stochastic algorithmic structure that can efficiently tackle CPD under a variety of non-Euclidean losses. Our development is an integrated design of data sampling and  {\it Lipschitz-like convexity}~\cite{bauschke2017descent} based local surrogate construction, leveraging the underlying multilinear structure of low-rank tensors.

		\section{Proposed Approach}
		\label{sec:smd}

		\subsection{Preliminaries}
		{
			A number of algorithmic frameworks have been considered for handling Problem~\eqref{eq:optori} with non-Euclidean losses, e.g., stochastic gradient descent (SGD) \cite{hong2020generalized}, block coordinate descent (or, alternating optimization (AO))~\cite{wang2020learning,chi2012tensors}, and the Gauss-Newton (GN) method~\cite{vandecappelle2020second}. 
			
			Let us denote $\bm{A}:=(\bm{A}_1,\bm{A}_2,\ldots,\bm{A}_N)$, $\mathcal{A}:=\mathcal{A}_1\times\mathcal{A}_2\times\ldots\times\mathcal{A}_N$. We also use $F(\bm{A})$ to represent the objective function of Problem~\eqref{eq:optori}. The updates of AO and SGD type algorithms can be summarized as follows:
			\begin{subequations}\label{eq:aosgd}
				\begin{align}
					(\mathtt{AO})\ \bm{A}_n^{t+1}&\approx \arg\min_{\bm{A}_n\in\mathcal{A}_n}\ F(\bm{A}_n;\bm{A}_{-n}^t),\label{eq:ao}\\
					(\mathtt{SGD})\ \bm{A}^{t+1}&=\textrm{Proj}_{\mathcal{A}}(\bm{A}-\eta_t \hat{\bm{G}}^t)\notag\\
					&=\arg\min_{\bm{A}\in\mathcal{A}}\ \langle \bm{A},\hat{\bm{G}}\rangle + \frac{1}{2\eta_t}\| \bm{A} - \bm{A}^t \|_F^2.\label{eq:sgd}
				\end{align}
			\end{subequations}
			In~\eqref{eq:aosgd}, $\bm{A}_{-n}$ corresponds to $\bm{A}$ with $\bm{A}_n$ being removed; $F(\bm{A}_n;\bm{A}_{-n}^t)$ is the objective function with fixed $\bm{A}_{-n}^t$; $\hat{\bm{G}}^t$ represents the {\it gradient estimation} from sampled data; $\eta_t>0$ is the step size; and $\textrm{Proj}_{\mathcal{A}}(\cdot)$ denotes the projection onto constraint set $\mathcal{A}$. 
			Many deterministic non-Euclidean tensor decomposition algorithms take the AO route; see, e.g., \cite{wang2020learning,chi2012tensors}, whereas the recent work in \cite{vandecappelle2020second} used a GN method to improve the iteration complexity (i.e., the number of iterations needed for reaching a satisfactory solution).
			However, the AO and GN methods face heavy per-iteration computational and memory complexities when handling large-scale tensors; see the ``MTTKRP'' challenge discussed in \cite{sidiropoulos2017tensor,fu2020block,fu2020nonconvex,kolda2009tensor}.

			The SGD approach in~\cite{hong2020generalized} is more lightweight in terms of the per-iteration resource consumption. In each iteration, the gradient estimation is computed as follows:
			\begin{equation}\label{eq:Gest}
				\hat{\bm G} = \frac{1}{|\mathcal{S}|}\sum_{\bm{i}\in\mathcal{S}}\nabla_{\bm{A}} \ell\left(\underline{\bm X}_{\bm i},{\sum\nolimits_{r=1}^R \prod\nolimits_{n=1}^N} \bm{A}_n(i_n,r)\right)
			\end{equation} 
			where $\mathcal{S}\subseteq\mathcal{I}$ is the sample index set used in this iteration, $\underline{\bm X}_{\bm i}$ for $\bm i\in {\cal S}$ is the sampled data, and $\nabla_{\bm{A}}$ denotes the operation of taking derivative with respect to $\bm{A}$. Constructing $\hat{\bm G}$ can be fairly economical since only partial data is used. This makes the per-iteration complexity of the algorithm affordable, even if the tensor of interest is large.
			
			However, simply using SGD for the non-Euclidean CPD problem may not be the most effective approach.  One can see that from~\eqref{eq:sgd}, every iteration of} SGD is equivalent to solving a quadratic program, which is used as a local surrogate of the original cost function. However, it is known that such quadratic functions may not be a good approximation for many non-Euclidean losses. In particular, using quadratic local surrogates may result in slow progresses since it largely ignores the geometry of the cost function ~\cite{beck2003mirror,bauschke2017descent}. This will become clearer later in Fig.~\ref{fig:cont} (see our detailed discussion in Remark~\ref{rmk:phi}).
		
		
		In this section, we will propose a stochastic mirror descent (SMD) framework to handle the non-Euclidean CPD problem. MD is able to take the problem geometry into consideration and thus could be substantially more efficient than GD under non-Euclidean cost functions, if properly designed. 
		Together with a tensor fiber sampling strategy advocated in \cite{battaglino2018practical,fu2020block}, the aforementioned challenges in constraint enforcing and convergence analysis will also be addressed.

	}
	
	\subsection{ Data Sampling}\label{subsec:fiber}
	A key ingredient for stochastic algorithms lies in the data sampling strategy, which uses partial data to estimate the direction-to-go in each iteration. Under the Euclidean loss, sub-tensor sampling~\cite{vervliet2015randomized}, random entry sampling~\cite{beutel2014flexifact}, and tensor fiber sampling~\cite{battaglino2018practical,fu2020block} were all considered---which all offered effective solutions. In principle, all the sampling strategies considered in the Euclidean case could still be used in the non-Euclidean cases. For example, the recent non-Euclidean CPD work in~\cite{hong2020generalized} used an entry sampling scheme.
	Nonetheless, since non-Euclidean losses are inherently more complex, different sampling strategies may lead to algorithms that admit drastically different updating rules and convergence properties.
	
	In this work, we advocate the fiber sampling strategy that was used in~\cite{battaglino2018practical,fu2020block} for Euclidean loss CPD;  see illustration in Fig.~\ref{fig:fiber} for tensor fibers. We find this sampling strategy particular handy in the non-Euclidean case, for a couple of reasons:
	\begin{figure}
		\centering
		\includegraphics[width=0.95\linewidth]{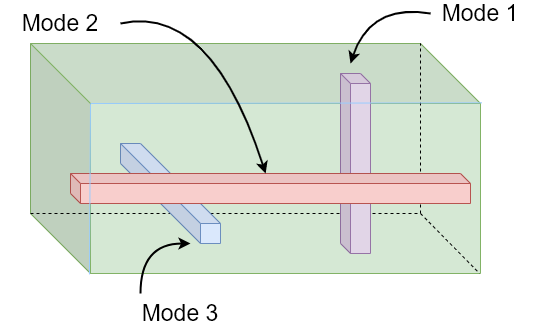}
		\caption{Mode-$i$ fibers of a third-order tensor, where $i=1,2,3$.}
		\label{fig:fiber}
	\end{figure}
	
	\noindent
	$\bullet$ {\bf Incorporating Prior on $\bm{A}_n$:} Randomly sampling some indexes $\bm{i}$~\cite{beutel2014flexifact}  or selecting a subtensor~\cite{vervliet2015randomized} face an issue that the sampled data may relate to only some rows of $\bm{A}_n$ [cf. Eq.~\eqref{eq:Gest}].
	This makes enforcing constraints on the columns of $\bm A_n$ impossible---while many important constraints under statistical non-Euclidean CPD are imposed on the columns of $\bm A_n$, e.g., the probability simplex constraints in \cite{kargas2019learning,yeredor2019estimation,yeredor2019maximum,ibrahim2020recovering}.
	Nonetheless, fiber sampling makes every batch of sampled data contains information about one full $\bm A_n$, making enforcing such constraints easy.
	
	\noindent
	$\bullet$ {\bf Convex Approximation to Optimize $\bm{A}_n$:}~The fiber sampling strategy also provides a way to further exploit the block-wise structure of the loss functions. 
	Even under complex non-Euclidean losses, with the notion of \textit{Lipschitz-like convexity}~\cite{bauschke2017descent}, the loss function $\ell(\cdot)$ with respect to each block ${\bm{A}}_n$ can often be locally approximated (or, to be precise, majorized) by a strongly convex function---which will prove useful for deriving lightweight updates and establishing convergence guarantees.

	The fiber sampling scheme can be understood using the \textit{matrix unfolding} representations of low-rank tensors. The mode-$n$ matrix unfolding of $\underline{\bm{X}}$ is a $J_n\times I_n$ matrix, denoted as $ \bm{X}_{n}$, and the entry-wise correspondence is 
	$\underline{\bm{X}}_{\bm{i}}=\bm{X}_n(j,i_n),\ j=1+\sum_{k=1,k\neq n}^N (i_k-1)J_k,$
	where $J_k=\prod_{m=1,m\neq n}^{k-1}I_m$. For the low-rank tensor $\underline{\bm{M}}$ in \eqref{eq:CPDmodel}, its mode-$n$ matrix unfolding can be expressed as 
	$\bm{M}_n=\bm{H}_n {\bm{A}}_n^T,$
	where $\bm{H}_n={\bm{A}}_N \odot \bm{A}_{N-1}\ldots\odot \bm{A}_{n+1}\odot \bm{A}_{n-1}\odot\ldots\odot \bm{A}_{1}$ and $\odot$ denotes the Khatri-Rao product.

	{Based on the mode-$n$ matrix unfolding for both $\underline{\bm{X}}$ and $\underline{\bm{M}}$, Problem~\eqref{eq:optori} with $\{A_m\}_{m\neq n}$ being fixed can be recast as follows:
		\begin{equation}\label{eq:opt_fiber}
			\begin{aligned}
				\min_{\bm{A}_n}\ \frac{1}{J_n I_n}\sum_{j=1}^{J_n} f_j^n(\bm{A}_n;\bm{H}_{n})\quad \mathrm{s.t.}\ \bm{A}_n\in\mathcal{A}_n,
			\end{aligned}
		\end{equation}
		where $f_j^n(\bm{A}_n;\bm{H}_{n})=\sum_{i=1}^{I_n} \ell ({\bm{X}}_n(j,i),{\bm{H}}_n(j,:)\bm{A}_n(i,:)^T)$. Each row of $\bm{X}_n$ is a mode-$n$ fiber. The mode-$n$ fiber sampling uses part of rows of $\bm{X}_n$ as well as the corresponding rows of $\bm{H}_n$.} Denote $\mathcal{F}_n\subset\{1,2,\ldots,J_n\}$ as the index set of the sampled fibers. Then, the sampled version of Problem~\eqref{eq:opt_fiber} becomes
	\begin{equation}\label{eq:subopt}
		\begin{aligned}
			\min_{\bm{A}_n\in\mathcal{\bm{A}}_n}\ &\frac{1}{|\mathcal{F}_n|I_n}\sum_{j=1}^{|\mathcal{F}_n|} \sum_{i=1}^{I_n} \ell ({\hat{\bm{X}}}_n(j,i),{\hat{\bm{H}}}_n(j,:)\bm{A}_n(i,:)^T),
		\end{aligned}
	\end{equation}
	where ${\hat{\bm{X}}}_n=\bm{X}_n(\mathcal{F}_n,:)$ and ${\hat{\bm{H}}}_n=\bm{H}_n(\mathcal{F}_n,:)$.

	{ Note that if $\ell$ is the Euclidean loss, the subproblem \eqref{eq:subopt}  is a (constrained) least squares problem, which is convex (if ${\cal A}_n$ is a convex set) and can be relatively easily solved~\cite{battaglino2018practical,fu2020block}.
		However, when $\ell$ is a non-Euclidean loss, the problem in \eqref{eq:subopt} may still be nonconvex ( e.g., $\beta$-divergence with $\beta<1$) and challenging. In addition, one hopes to solve \eqref{eq:subopt} using economical updates, which is often an art when non-Euclidean losses are considered.}

	\subsection{Block-Wise Approximation via Bregman Divergence }\label{subsec:cvxup}
	

	To derive a lightweight update of $\bm A_n$ from \eqref{eq:opt_fiber},
	note that all the loss functions $\ell(x,m)$'s given  in Table~\ref{tab:ell} can be decomposed into a convex part $\check{\ell}(x,m)$ plus a concave part $\hat{\ell}(x,m)$ as  $\ell(x,m)=\check{\ell}(x,m)+\hat{\ell}(x,m)$. To see how to make use of such a convex-concave property, let us consider the $(i,j)$th component in \eqref{eq:subopt}, and simplify it as $\ell(x,\bm{h}^T\bm{a})$, where we have defined $x={\hat{\bm{X}}}_n(j,i),\ \bm{h}^T=\hat{\bm{H}}_n(j,:),$ and $\bm{a}=\bm{A}_n(i,:)^T$.  The following lemma constructs a strongly convex surrogate function of $\ell(x,\bm{h}^T\bm{a})$:
	
	\begin{figure*}[ht]
		\centering
		\includegraphics[width=0.95\linewidth]{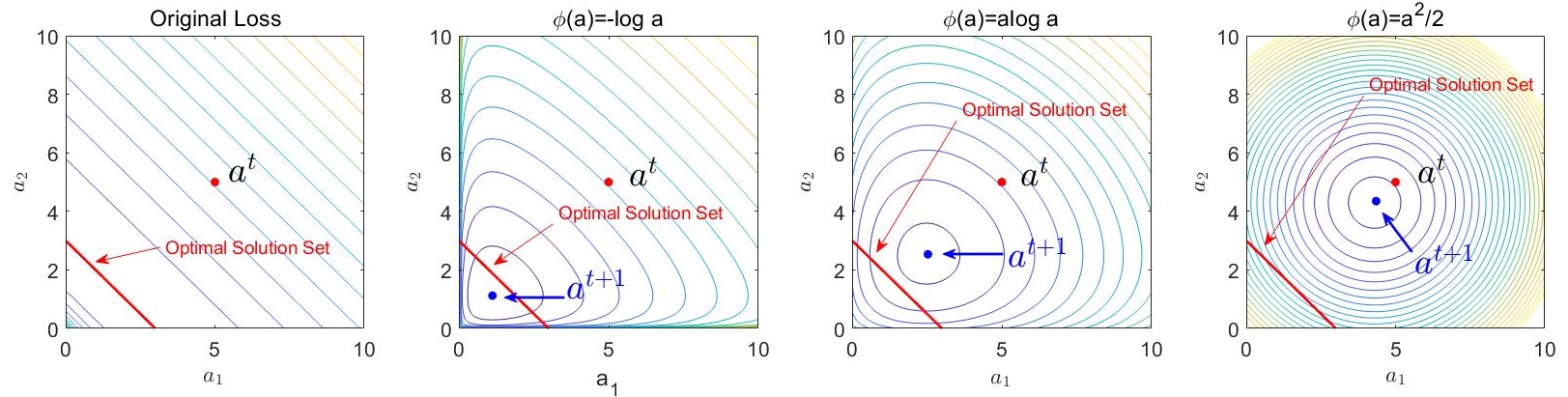}
		\caption{Contour maps of upper bound functions (except a same constant) indicated by~\eqref{eq:lem-ell} with different $\phi(a)$.  The loss function $\ell(x,m)=m-x\log (m)$, where $x=3$, $\epsilon=10^{-9}$, $m=\bm{h}^T\bm{a}$ with $\bm{h}=(1,1)$ and $\bm{a}=(a_1,a_2)$. The solid red point is $\bar{\bm{a}}=(5,5)$ and the solid red line denotes the solution set, $\bm{h}^T\bm{a}=x-\epsilon$.}
		\label{fig:cont}
	\end{figure*}
	
	\begin{lem}\label{lem:phi}
		Suppose for a given $\bm{h}$, there exists $\bar{\bm{a}}$  such that $h_r,\bar{a}_r>0,\forall r$.
		Let $\phi(\cdot):\mathbb{R}\mapsto\mathbb{R}$ be a strongly convex function satisfying the following condition 
		\begin{equation}\label{eq:cond} 
			\exists L<\infty \  \textrm{such that}\  L\phi(a_r)-\lambda_r\check{\ell}\left(x,\frac{h_r}{\lambda_r}{a}_r\right)\textrm{ is convex, }\ \forall r,
		\end{equation}
		where $\lambda_r :=\frac{h_r \bar{a}_r}{\bm{h}^T\bar{\bm{a}}}>0,\forall r$. 
		Let $D_\phi(\bm{a}, \bar{\bm{a}})$ be the Bregman divergence generated by $\phi(\cdot)$:
		\begin{equation*}
			D_\phi(\bm{a}, \bar{\bm{a}})=\sum_{r=1}^R\phi(a_r) - \phi(\bar{a}_r)-\langle\nabla\phi(\bar{a}_r), a_r-\bar{a}_r \rangle.
		\end{equation*}
		Then the following holds:
		\begin{equation}\label{eq:lem-ell}
			\ell(x,\bm{h}^T\bm{a})\leq \ell(x,\bm{h}^T\bar{\bm{a}})+\langle\nabla \ell(x,\bm{h}^T\bar{\bm{a}}), \bm{a}-\bar{\bm{a}}  \rangle + L D_\phi(\bm{a}, \bar{\bm{a}}),
		\end{equation}
		and equality holds if and only if $\bm{a}=\bar{\bm{a}}$. 
	\end{lem}

	\begin{proof}
		See Appendix~\ref{app:phi} in Supplementary Materials.
	\end{proof}
	
	{The upper bound constructed in Lemma~\ref{lem:phi} is reminiscent of the majorization-minimization (MM) scheme developed in for $\beta$-divergence; see, e.g., ~\cite{fevotte2011algorithms}. The difference here is the choice for $\phi(\cdot)$. The MM scheme in~\cite{fevotte2011algorithms} suggested to choose $\phi(\cdot)$ to be the convex part of $\ell(\cdot)$ (up to a constant scaling) while Lemma~\ref{lem:phi} suggests to choose $\phi(\cdot)$ to `fit' the geometry of the convex part of $\ell(x,\bm{h}^T\mathbf{a})$. Clearly, the condition in \eqref{eq:cond} is more general. The corresponding functions $\left(\phi(\cdot),\check{\ell}(x,\cdot)\right)$ are referred to as the \textit{Lipschitz-like convexity} function pair~\cite{bauschke2017descent}.}

	Lemma \ref{lem:phi} suggests that if one can find the appropriate $\phi(\cdot)$, then problem~\eqref{eq:subopt} can be approximately solved via {the following update:}
	\begin{equation}\label{eq:updateA}
		\begin{aligned}
			\bm{A}_n^{t+1}=\arg\min_{\bm{A}\in\mathcal{\bm{A}}_n}\  \langle {\hat{\bm{G}}}^t_n, \bm{A} - \bm{A}_n^t\rangle+\frac{1}{\eta_t} D_{\phi}(\bm{A} , \bm{A}_n^t),
		\end{aligned}
	\end{equation}
	where
	$D_{\phi}(\bm{A} , \bm{A}_n^t)$ is the Bregman divergence between $\bm{A}_n$ and $\bm{A}_n^t$ defined as 
	\begin{equation}\label{eq:Dphi}
		\begin{aligned}
			D_{\phi}(\bm{A} , \bm{A}_n^t)=&\sum_{i,j}\phi(\bm{A}(i,j)) - \phi(\bm{A}_n^t(i,j))\\
			&-\langle\nabla\phi(\bm{A}_n^t(i,j)), \bm{A}(i,j)-\bm{A}_n^t(i,j) \rangle.
		\end{aligned}
	\end{equation}
	For different $\ell(x,m)$ in Table~\ref{tab:ell}, the expressions of ${\hat{\bm{G}}}^t_n$ are summarized in Table~\ref{tab:G} in Supplementary Materials. 

	
	
	The step in~\eqref{eq:updateA} is an SMD update for the block variable $\bm A_n$. We refer the readers to the optimization literature in \cite{dang2015stochastic,hanzely2018fastest,zhang2018convergence,lu2019relative} for generic SMD. The term $\langle {\hat{\bm{G}}}^t_n, \bm{A} - \bm{A}_n^t\rangle$ captures the first-order information and $\frac{1}{\eta_t}D_{\phi}(\bm{A},\bm{A}_n^t)$ in~\eqref{eq:updateA} can be regarded as a geometry-aware augmentation for properly approximating the loss in~\eqref{eq:subopt}. 
	If choosing $L$ in~\eqref{eq:cond} as $\frac{1}{\eta_t}$ such that \eqref{eq:cond} holds, then \eqref{eq:updateA} minimizes an upper bound function for the loss in \eqref{eq:subopt}. In addition, considering $\phi(a)=\frac{1}{2}\|a \|^2$, then~\eqref{eq:updateA} reduces to the well-known projected SGD.
	
	\begin{rmk}\label{rmk:phi}
		For non-Euclidean losses, the choice of $\phi(\cdot)$ can heavily affect the behavior of the algorithm. 
		On the other hand, the flexibility of using different  $\phi(\cdot)$'s also entails opportunities of developing fast non-Euclidean CPD algorithms.
		An illustrative example using the generalized KL loss is shown
		in Fig.~\ref{fig:cont}. One can see that using $\phi(a)=a^2/2$---which corresponds to the gradient descent in \eqref{eq:Gest}---the progress from $a^{t}$ to $a^{t+1}$ is very small. However, by using $\phi(a)$'s that are more adapted to the cost function's geometry (reflected by the contour of the cost function), the progress in one iteration can be much larger.
	\end{rmk}

	%

	
	\subsection{Stochastic Mirror Descent for CPD}\label{subsec:prop}
	
	The SMD step in~\eqref{eq:updateA} specifies the update for one latent matrix $\bm{A}_n$. Combining it with a random selection of block $n$, the proposed algorithm is summarized in Algorithm~\ref{alg:smd}. 
	The major advantage of using random block selection is that such scheme leads to an unbiased gradient estimation~\cite{fu2020block} (up to a constant scaling), which simplifies the convergence analysis.
	
	In essence, the proposed algorithm is a block-randomized (inexact) coordinate descent method,
	which admits a similar structure as the algorithm in \cite{fu2020block} for CPD under the Euclidean loss.
	The key difference is that Algorithm~\ref{alg:smd} employs SMD to solve each subproblem inexactly, while the algorithm in \cite{fu2020block} uses proximal gradient.
	For simplicity, we name it as Stochastic Mirror descent AlgoRiThm for CPD (\texttt{SmartCPD}).

	
	
	\begin{rmk}
		If no data sampling or block sampling is considered, the full batch version of
		\texttt{SmartCPD} subsumes many
		existing non-Euclidean and Euclidean matrix/tensor decomposition algorithms as its special cases---see the algorithms in \cite{fevotte2011algorithms,phan2013fast,chi2012tensors,ermics2015link,kargas2019learning,fu2020block,hien2020algorithms}. In particular, consider the KL divergence. If one chooses $\phi(a)=-\log a$ and choose the step size $\eta_t$ properly, \texttt{SmartCPD} becomes the MM algorithm~\cite{chi2012tensors}; when one uses $\phi(a)=a\log a$, then \texttt{SmartCPD} becomes the MD algorithm developed in~\cite{kargas2019learning}.
		This connection is not surprising, since MD includes many first-order approaches as its special cases. {Nonetheless, using this connection, our convergence analysis (cf. Sec.~\ref{sec:converge}) may also shed some light on the convergence behaviors of some existing algorithms whose convergence analyses were not considered at the time (e.g., \cite{ermics2015link}).}
	\end{rmk}

	\begin{algorithm}
		\caption{Stochastic Mirror Descent (SMD) Algorithm}
		\label{alg:smd}
		\small
		\begin{algorithmic}[1]
			\Require $X,\bm{A}_1^0,\bm{A}_2^0,\ldots,\bm{A}_N^0, \phi, \{\eta_t\}_{t=0,1,\dots}$
			\For{$t=0,1,\ldots,$ until meet some convergence criteria}
			\State Uniformly sample $n\in\{1,2,\ldots,N\}$;
			\State Uniformly sample fibers $\mathcal{F}_n\subset\{1,2,\ldots,J_n\}$;
			\State Compute the sampled gradient ${\hat{\bm{G}}}_n^t$;
			\State $\bm{A}_n^{t+1}=\arg\min_{\bm{A}\in\mathcal{\bm{A}}_n} \langle {\hat{\bm{G}}}_n^t, \bm{A} - \bm{A}_n^t\rangle +\frac{1}{\eta_t}D_{\phi}(\bm{A} , \bm{A}_n^t)$;
			\State $\bm{A}_i^{t+1} = \bm{A}_i^{t},\forall i\neq n$;
			\EndFor
		\end{algorithmic}
	\end{algorithm}
	

	%

	
	\subsection{Practical Implementation}
	\label{sub:L}
	
	To implement the \texttt{SmartCPD} algorithm in practice, a number of key aspects need to be considered carefully.
	In particular, as in all stochastic algorithms, the step size-related parameter selection (i.e., $\eta_t$ and $L$ in \texttt{SmartCPD}) needs to be carefully carried out. In addition, the $\phi(\cdot)$ function should be chosen judiciously. In this subsection, we discuss these aspects in detail.

	\subsubsection{Choice of $\phi(\cdot)$} 
	As indicated by Lemma~\ref{lem:phi}, $\phi(\cdot)$ should be chosen to adapt the geometry of $\ell(\cdot)$---e.g., by setting $\phi(\cdot)$ to be the convex part of $\ell(\cdot)$ if $\ell(\cdot)$ has convex-concave structure. In addition, the update~\eqref{eq:updateA} needs to solve a subproblem which minimizes the constructed surrogate loss over constraint set $\mathcal{\bm{A}}_n$. Practically, it is desirable that this subproblem can be solved easily, preferably in closed form. Hence, the choice of $\phi(\cdot)$ should be an integrated consideration of the function geometry of $\ell$ and the constraint ${\cal A}_n$. For example, consider $\ell(x,m)=x\log m$ with $\mathcal{\bm{A}}_n$ being the probability simplex constraint, i.e., $\bm{A}_n^T\bm{1}=\bm{1},\bm{A}_n(i,j)\geq 0,\forall i,j$. Then, it is preferred to choose $\phi(a)=a\log a$ other than $\phi(a)=-\log a$ since the former admits a closed-form solution of the MD update, which is also known as the {\it exponential gradient descent} or {\it entropic descent}~\cite{beck2003mirror}. Some choices of $\phi$ for non-Euclidean $\ell$'s under various ${\cal A}_n$'s are summarized in Table~\ref{tab:solu}. 
	\begin{table}
		\centering
		\caption{Some examples for pair $(\phi,\mathcal{\bm{A}}_n)$ and closed-form solution of~\eqref{eq:updateA}.}
		\resizebox{\linewidth}{!}{	\begin{threeparttable}
				\begin{tabular}{c c c}
					\toprule
					$\phi(\cdot)$  & $\mathcal{\bm{A}}_n$ & \textbf{Closed-form Solution}\\
					\midrule
					$-\log a$ & non-negative & $\bm{A}_n^t\circledast\left[\hat{\bm{G}}_n^t\circledast (\bm{A}_n^t\oslash L) + 1 \right]$\\
					$a\log a$ &  non-negative & $\bm{A}_n^t\circledast\exp(-\hat{\bm{G}}_n^t \oslash L)$\\
					$a^c$ ($c>1$ or $c<0$) & non-negative & $\left[(\bm{A}_n^t)^{.(c-1)} - \frac{1}{c}\hat{\bm{G}}_n^t\oslash L\right]^{.\frac{1}{c-1}} $ \\
					$a\log a$ &  simplex & $ \mathrm{colnorm}(\bm{A}_n^t\circledast\exp(-\hat{\bm{G}}_n^t \oslash L))$\\
					$a^2$ & many forms &  refer~\cite{parikh2014proximal,fu2020block}\\
					\bottomrule
				\end{tabular}
				\begin{tablenotes}
					\footnotesize
					\item {\scriptsize *In this table, $\mathrm{colnorm}(\bm{A})=\bm{A}\oslash 1(\bm{A}^T1)^T$, which denotes the column-wise normalization operation.}
				\end{tablenotes}
		\end{threeparttable}}
		\label{tab:solu}
	\end{table}
	
	\subsubsection{Choice of $\eta_t$} \label{sec:choiceL}
	Since the Bregman divergence $D_{\phi}(\bm{A},\bm{A}^\prime)$ in~\eqref{eq:Dphi} is defined in an entry-wise summation form, {scaling $D_{\phi}(\bm{A},\bm{A}^\prime)$ with a constant $\frac{1}{\eta_t}$ for all entries may be less effective for approximating the loss in~\eqref{eq:subopt}.
		In this work, we propose to use coordinate-dependent step size, i.e.,
		different positive scaling factors $\bm{\Gamma}_n^t(i_n,j)$'s for different coordinates $(i_n,j)$'s, where $i_n\in[I_N]$, $j\in[R]$, and $\bm{\Gamma}_n^t\in\mathbb{R}^{I_n\times R}$. Consequently, the term $\frac{1}{\eta_t} D_{\phi}(\bm{A} , \bm{A}_n^t)$ in~\eqref{eq:updateA} becomes 
		\begin{equation}\label{eq:updateA2}
			\begin{aligned}
				\sum_{i_n,j}\bm{\Gamma}_n^t(i_n,j)D_{\phi}(\bm{A}_n(i_n,j),\bm{A}_n^t(i_n,j)).
			\end{aligned}
		\end{equation}
		As illustrated in Fig.~\ref{fig:empL}, such a coordinate-specific $\bm{\Gamma}_n^t$ empirically helps accelerate convergence. In the what follows, we provide two schemes on choosing $\bm{\Gamma}_n^t(i_n,j)$'s}.
	
	
	
	\begin{figure*}[ht]
		\centering
		\subfigure[Continuous Data Case (diff. step size  and $\phi$)]{
			\includegraphics[width=0.31\linewidth]{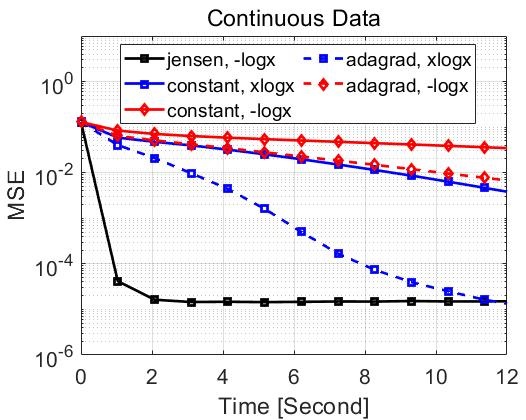}
			\label{subfig:La}}
		\subfigure[Count Data Case (diff. step size and $\phi$)]{
			\includegraphics[width=0.31\linewidth]{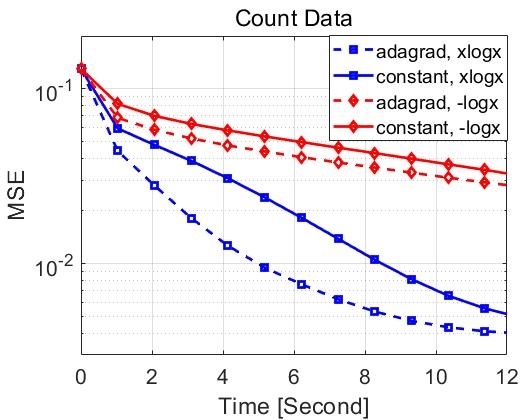}
			\label{subfig:Lb}}
		\subfigure[Continuous Data Case (diff. no. of iteration)]{
			\includegraphics[width=0.31\linewidth]{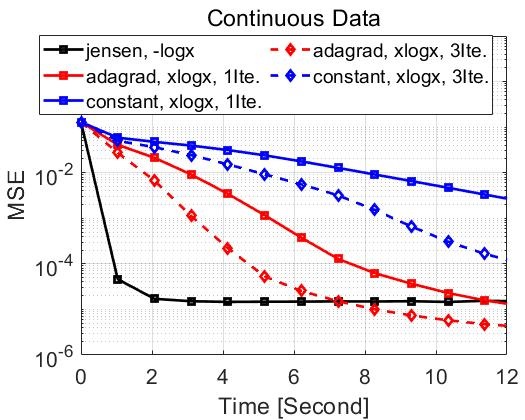}
			\label{subfig:Lc}}
		\caption{Averaged MSE over $20$ independent trials. The latent matrices are draw from i.i.d. uniform distribution between $0$ and $1$. For continuous data case (Figs.~\ref{subfig:La} and~~\ref{subfig:Lc}), Gaussian noise with SNR$=40$dB (see definition in Section~\ref{subsec:simul_continuous}) is added in the data; in count data case (Fig.~\ref{subfig:Lb}), each tensor entry is drawn from Poisson distribution with identity link function.}
		\label{fig:empL}
	\end{figure*}
	
	\noindent
	$\bullet$ {\bf Jensen's inequality based choice}: 
	{One important criteria is to choose $\bm{\Gamma}_n^t(i_n,j)$ such that the inequality in~\eqref{eq:lem-ell} holds for each coordinate $(i_n,j)$. Such a choice makes the local approximation in \eqref{eq:updateA} an upper bound of~\eqref{eq:subopt}.
		In the case where $\ell(\cdot)$ enjoys the convex-concave property and the constraint sets ${\cal A}_n$ satisfy $\mathbb{R}_+^{I_n\times R} \subseteq \mathcal{\bm{A}}_n,\forall n$, Lemma~\ref{lem:phi} implies that one can choose $\phi(\cdot)$ to be the convex part of $\ell(\cdot)$. Consequentially, $\bm{\Gamma}_n^t(i_n,j)$ can be derived based on the Jensen's inequality in each iteration. A number of examples of $\bm{\Gamma}_n^t(i_n,j)$ with respect to different $\ell(\cdot)$ are given in Table~\ref{tab:jensen}, where the derivation are based on~\eqref{eq:js} and \eqref{eq:lcvx_up} in Supplementary Materials. 
		We note that the Jensen's inequality based $\bm{\Gamma}_n^t(i_n,j)$ choice is often used 
		by popular methods such as MM~\cite{sun2016majorization} and block successive upper bound minimization (BSUM)~\cite{hong2015unified} to simplify the solution of the subproblems. 
		The difference here is our construction uses sampled fibers instead of the whole data.  Such Jensen's inequality-based step-size choice also works well with certain non-Euclidean losses under our stochastic settings, as will be seen in the experiments.
	}
	
	

	

	\noindent
	$\bullet$ {\bf Adaptive step size based choice}: The deep learning community has developed a number of effective adaptive step size scheduling methods, e.g., the Adagrad~\cite{duchi2011adaptive} and Adam type schemes~\cite{chen2018convergence}.
	These schemes typically exploit the past iterations' gradient information to scale the current sampled gradient in a coordinate-wise manner. 
	The upshot of these methods is that they often require very small amount of step size tuning, yet offer highly competitive empirical performance; also see theoretical understanding in \cite{chen2018convergence}.
	Under the Bregman divergence, the adaptive step size schemes can be used to `schedule' $\bm{\Gamma}_n^t(i_n,j)$ in each iteration. Specifically, we propose the following Adagrad step size rule for $\bm{\Gamma}_n^t(i_n,j)$,
	\begin{equation}\label{eq:szL}
		\bm{\Gamma}_n^t(i_n,j)=\sqrt{\sum_{t^\prime=0}^{t-1}[\hat{\bm{G}}_n^{t^\prime}(i_n,j)]^2+b},
	\end{equation}
	where $\hat{\bm{G}}_n^{t^\prime}(i_n,j)$ denotes the $(i_n,j)$-th entry of $\hat{\bm{G}}_n^{t^\prime}$ and $b> 0$ is a constant to ensure $\bm{\Gamma}_n^t(i_n,j)>0$. Note that the adaptive step size was considered in fiber-sampling based stochastic Euclidean CPD in \cite{fu2020block} and entry-sampling based stochastic non-Euclidean CPD \cite{hong2020generalized}, and encouraging results were observed in both cases.
	
	
	

	\begin{rmk}
		In Fig.~\ref{fig:empL}, a numerical example on a $100\times100\times100$ tensor with rank $R=10$ is presented,  where the generalized KL divergence is selected as the loss function and the averaged mean squared error (MSE, see definition in \cite{fu2020block}) of the latent matrix is used as performance metric. It can be observed that for continuous data (Fig.\ref{subfig:La}), using different $\bm{\Gamma}_n(i_n,j)>0$ for different coordinate $(i_n,j)$ exhibits much faster convergence behavior than using a constant step size for all coordinates. 
		For count data (Fig.~\ref{subfig:Lb}), we have similar observations. 
		
		Another observation from Fig.~\ref{fig:empL} is that, for count or binary data, the Jensen's inequality based step size scheme may be less competitive, especially when the data contains many zeros. 
		The zero entries may make $\bm{\Gamma}_n(i_n,j)$ very small, thereby causing numerical issues. The 'adagrad' scheme in~\eqref{eq:szL} has empirically much more stable convergence behavior under such circumstances.
	\end{rmk}

	\subsubsection{Inner iterations} 
	The proposed Algorithm~\ref{alg:smd} only contains one iteration per block.
	Nonetheless, one can also extend it to multiple SMD updates (i.e., inner iterations) per block.
	We have observed that implementing with a few more inner iterations could improve the practical convergence behavior---as shown in Fig.~\ref{subfig:Lc}. There are two ways of having multiple inner iterations. The first way is to repeat lines 3-5 in Algorithm~\ref{alg:smd} to update $\bm{A}_n$ several times before moving to the next block, where the fibers are re-sampled for each inner iteration. The second way is that, for fixed sampled fibers, repeat lines 4-5 for multiple times. Both methods work reasonably well in practice. In this work, we use latter because it is more sample efficient. 

	\section{Convergence Analysis }
	\label{sec:converge}

	In general, convergence guarantees of stochastic tensor decomposition algorithms are difficult to establish, as nonconvex constrained optimization problems are intrinsically harder to analyze under stochastic settings. 
	The non-Euclidean version is even more so, since the sampled subproblems may still be nonconvex. There has not been an analytical framework for SMD based non-Euclidean tensor/matrix factorization.
	The nonconvex block SMD in the optimization literature \cite{dang2015stochastic} is the most closely related to our algorithmic framework.
	However, the convergence analysis there does not cover the proposed algorithm. Specifically, the convergence of the algorithm in \cite{dang2015stochastic} hinges on some special {\it incremental block averaging steps}, which is not used in our algorithm.
	More importantly, the algorithm in \cite{dang2015stochastic} requires that the block-wise gradient estimation error vanishes to zero when the iterations progress. This may require implementing the algorithm with an active variance reduction technique, e.g., increasing the batch size in each iteration \cite{fu2020block},
	which is not entirely realistic, and it is somewhat against the purpose of using stochastic algorithms.

	Our goal is to offer tailored convergence analysis for \texttt{SmartCPD} that does not rely on conditions like incremental block averaging or vanishing gradient estimation error.
	We note that
	for constrained problems, convergence analysis for SMD with adaptive step size $\bm{\Gamma}_n^t$ is very challenging. Theoretical understanding of SGD with adaptive step size scheme was recently discussed in~\cite{chen2018convergence}---but the SMD case is still an open problem.
	In this work, we focus on the case where $\bm{\Gamma}_n^t(i_n,j)=\eta_t>0$ are all identical.

	Our analysis leverages the recently proposed notion of \textit{relative smoothness}~\cite{lu2018relatively}. Relative smoothness was used in
	several recent works to analyze single-block SMD type algorithms' convergence; see \cite{lu2019relative,davis2018stochastic,zhang2018convergence}. 
	Our analysis shares insights with these prior works and generalizes to cover the multi-block \texttt{SmartCPD}. Particularly, the proof takes advantage of the multilinear low-rank tensor structure and the block-randomized fiber sampling strategy to brige the gap between the single-block and multi-block cases.

	
	
	
	The objective function in \eqref{eq:optori} is denoted as $F(\bm{A})$ with $\bm{A}=(\bm{A}_1,\bm{A}_2,\ldots,\bm{A}_N)$. Then, Problem \eqref{eq:optori} can be re-expressed as: 
	\begin{equation}\label{eq:optref}
		\begin{aligned}
			\min_{\bm{A}}\ F(\bm{A}) + h(\bm{A})
		\end{aligned}
	\end{equation}
	where $h(\bm{A})=\sum_{n=1}^Nh_n(\bm{A}_n)$ and $h_n(\bm{A}_n)$ is the indicator function of set $\mathcal{A}_n$, i.e., $h_n(\bm{A}_n)=0$ if $\bm{A}_n\in\mathcal{A}_n$ and otherwise $h_n(\bm{A}_n)=\infty$. 
	
	
	
	
	Our first observation is as follows:
	\begin{lem}\label{lem:L}
		If $\{ \bm A_n^t\}$ all reside in a compact set for all $n$, there exists $0<L<\infty$ such that for any \ $\bm{A}^\prime,\bm{A}\in\mathcal{\bm{A}}$ we have,
		\begin{equation}
			\label{eq:asp_cond}
			\begin{aligned}
				|F(\bm{A}^\prime)-F(\bm{A})-\langle \nabla F(\bm{A}), \bm{A}^\prime-\bm{A}\rangle|
				\leq  L D_\phi(\bm{A}^\prime,\bm{A}),
			\end{aligned}
		\end{equation}
		where $\mathcal{\bm{A}}=\mathcal{\bm{A}}_1\times\mathcal{\bm{A}}_2\ldots\times\mathcal{\bm{A}}_N$ and $\phi(\cdot)$ can be any strongly convex function.
	\end{lem}
	\begin{proof}
		See Appendix~\ref{app:L} in the supplementary material.
	\end{proof}
	It is easy to see Lemma~\ref{lem:L} immediately implies inequality~\eqref{eq:lem-ell} in Lemma~\ref{lem:phi}. {Lemma~\ref{lem:L} is more general than inequality~\eqref{eq:lem-ell} since it has no restriction on the choice of $\phi(\cdot)$. This brings much convenience for the practical usage of Algorithm~\ref{alg:smd}. More important, Lemma~\ref{lem:L} holds for the all the optimization variables (as opposed to a block as in Lemma~\ref{lem:phi}) and indicates both lower and upper bound functions. This is important for establishing convergence guarantees.} Also, we note that inequality~\eqref{eq:asp_cond} is a generalization of standard Lipschitz-continuous gradient property of $F(\bm{A})$ under Bregman divergence, which is also known as \textit{relative smoothness}~\cite{lu2018relatively}. We  remark that the assumption that $\{ \bm A_n^t\}$ live in a compact set is not easy to check in advance, yet it is not hard to satisfy. When the constraints ${\cal A}_n$ are compact sets, then this assumption is naturally satisfied per the defined update in~\eqref{eq:updateA}. Even if ${\cal A}_n$ is not bounded, unbounded iterates are rarely (if not never) observed in our extensive numerical experiments.

	Based on Lemma~\ref{lem:L}, convergence of the proposed Algorithm~\ref{alg:smd} can be guaranteed for any strongly convex function $\phi(\cdot)$ (see Theorem~\ref{thm}). 
	{{Our convergence analysis starts by using the following ``reference'' function in each iteration:
			\begin{equation*}
				\mathcal{L}(\bm{A};\bm{A}^\prime):=F(\bm{A})+h(\bm{A})+\frac{1}{2\lambda}D_\phi(\bm{A},\bm{A}^\prime)
			\end{equation*}
			where {{$0<\lambda<\frac{1}{2L}$}} is a constant such that $\mathcal{L}(\bm{A};\bm{A}^\prime)$ is strongly convex in $\bm{A}$. The minimal value and minimizer of $\mathcal{L}(\bm{A};\bm{A}^\prime)$ for a given $\bm{A}^\prime$, which are also known as \emph{Bregman Moreau envelope} $\mathcal{M}(\bm{A}^\prime)$ and \emph{Bregman proximal mapping} $\mathcal{T}(\bm{A}^\prime)$, are defined as 
			\begin{equation*}
				\begin{aligned}
					\mathcal{M}(\bm{A}^\prime)=\min_{\bm{A}}\ \mathcal{L}(\bm{A};\bm{A}^\prime),\
					\mathcal{T}(\bm{A}^\prime)=\arg\min_{\bm{A}}\ \mathcal{L}(\bm{A};\bm{A}^\prime).
				\end{aligned}
			\end{equation*}
			Denote $\hat{\bm{A}}=\mathcal{T}(\bm{A})$ for a given $\bm{A}$,  we use the following lemma to show that $D_\phi(\hat{\bm{A}},\bm{A})$ is a measure for attaining stationary points of Problem~\eqref{eq:optref}.
			
			\begin{lem}\label{lem:station}
				$\bm{A}$ is a stationary point of Problem~\eqref{eq:optref}, i.e.,  $\bm{0}\in \nabla F(\bm{A})+\partial h(\bm{A})$, where $\partial h(\bm{A})$ denotes the subgradient, if and only if $D_{\phi}(\hat{\bm{A}},\bm{A})=0$.
			\end{lem}
			\begin{proof}
				see Appendix~\ref{app:sta} in the supplementary material.
			\end{proof}
			
			The key step for analyzing Algorithm~\ref{alg:smd} is to quantify its one iteration behavior:
			\begin{lem}\label{lem:descent}
				Suppose that $\{ \bm A_n^t\}$ all reside in a compact set for all $n$, and that $\phi$ is a $\sigma$-strongly convex function. Let $\bm{A}^{t+1}$ be generated by Algorithm~\ref{alg:smd} at iteration $t$, then we have 
				\begin{equation}\label{eq:desc_lem}
					\begin{aligned}
						&\mathbb{E}\left[ \mathcal{M}(\bm{A}^{t+1}) \right]\\
						&\leq \mathcal{M}(\bm{A}^t)-c_1\eta_t D_\phi(\bm{\hat{A}}^t,\bm{A}^t)+c_2\eta_t^2\mathbb{E}\left[\|\hat{\bm{G}_n}^t\|^2 \right],
					\end{aligned}
				\end{equation}
				where $c_1=\frac{(1-2\lambda L)}{4\lambda^2N}$ and $c_2=\frac{1}{4\lambda\sigma}$. The expectation in \eqref{eq:desc_lem} is taken over the random variable responsible for fiber sampling in iteration $t$.
			\end{lem}
			\begin{proof}
				See Appendix~\ref{app:desc} in the supplementary material.
			\end{proof}
			
			The lemma implies that, in expectation, if the step size $\eta_t$ is properly chosen, then $\mathcal{M}(\bm{x}^{t})$ decreases after every iteration. 
			Using Lemma~\ref{lem:descent} as a stepping stone, we show our main result:
			
			
			\begin{thm}\label{thm}
				Suppose that the assumptions in Lemma~\ref{lem:descent} hold.
				
				\noindent
				$1)$ for diminishing step size $\eta_t$ satisfying $\sum_{t=1}^{\infty}\eta_t=\infty$ and $\sum_{t=1}^{\infty}\eta_t<\infty$, Algorithm~\ref{alg:smd} converge to a stationary point in expectation, $$\liminf_{t\rightarrow\infty}\mathbb{E}\left[D_\phi(\hat{\bm{A}}^t,\bm{A}^t) \right]=0;$$
				
				\noindent
				$2)$ for a constant step size $\eta_t=\frac{1}{\sqrt{T}}$, $\frac{1}{\sqrt{T}}$-stationary solution of Problem~\eqref{eq:optref} can be obtained by Algorithm~\ref{alg:smd} within $T$ iterations, $$\min_{1\leq t\leq T}\mathbb{E}\left[D_{\phi}(\hat{\bm{A}}^t,\bm{A}^t) \right]\leq C/\sqrt{T},$$ where $C>0$ is a constant. The expectations are taken over all random variables for fiber and block sampling in all iterations jointly.
			\end{thm}
	}}
	
	\begin{proof}
		See Appendix~\ref{app:thm} in the supplementary material.
	\end{proof}
	
	\section{Simulations}
	\label{sec:simul}
	We use synthetic and real data experiments to showcase the effectiveness of \texttt{SmartCPD}.
	\subsection{Synthetic Data}
	We evaluate the numerical performance of the proposed SMD algorithm on different types of synthetic data, i.e., continuous, count, and binary data, under various non-Euclidean losses in Table~\ref{tab:ell} ($\epsilon$ in Table~\ref{tab:ell} is set to be $10^{-9}$).

	\subsubsection{Baselines and Performance Metric}
	We use two recent competitive algorithms as our main baselines. The first one is an entry-sampling based stochastic non-Euclidean CPD optimization algorithm, namely, \texttt{GCP-OPT}, proposed in~\cite{kolda2020stochastic}. The second baseline is the generalized Gauss-Newton (\texttt{GGN}) method for non-Euclidean CPD  ~\cite{vandecappelle2020second}. The \texttt{GCP-OPT} method is implemented in Tensor Toolbox~\cite{bader2017matlab} and `Adam' is selected as the optimization solver. 
	The \texttt{GGN} method~\cite{vandecappelle2020second} is implemented by \texttt{nlsb$\_$gndl} and the `preconditioner' is set as `block-Jacobi'; see more details in the Tensorlab toolbox~\cite{vervliet2016tensorlab}. For the proposed \texttt{SmartCPD}, different choices of step size $\bm{\Gamma}_n^t$ and function $\phi(\cdot)$ will be specified according to data type.
	
	
	We generate third-order tensors with different sizes and ranks, each dimension of the tensor keeps the same as $I=I_1=I_2=I_3$. For the two stochastic algorithms, $2\times R$ fibers ($2I_nR$ entry samples) are used per iteration.
	The MSE of the latent matrices is used as performance metric \cite{fu2020block}.

	\subsubsection{Count Data}\label{subsec:simul_count}
	We first evaluate the performance on count data tensors. For the proposed \texttt{SmartCPD}, $\bm{\Gamma}_n^t$ is scheduled following \eqref{eq:szL}. We use $\phi(a)=a\log a$ in the SMD. The constant $b$ in \eqref{eq:szL} is $10^{-5}$ for all simulation trials.

	We use the loss function $\ell(x,m)=m-x\log(m+\epsilon)$ as in \cite{chi2012tensors} and draw the latent matrices $\bm{A}_1,\bm{A}_2$, and $\bm{A}_3$ from i.i.d. uniform distribution between $0$ and $A_{\textrm{max}}$, where $A_{\textrm{max}}=0.5$ is a positive constant. For each column of the latent matrices, $5\%$ elements are randomly selected and replaced by i.i.d. samples from uniform distribution between $0$ and $10A_{\textrm{max}}$. This way, the elements have more diverse scales. The observed count data tensor $\underline{\bm{X}}$ is generated follow the Poisson distribution, i.e., $\underline{\bm{X}}_{\bm{i}}\sim \textrm{Poisson}(\underline{\bm{M}}_{\bm{i}})$. We set $\mathcal{\bm{A}}_n$ for all $n$ as the nonnegative orthant.

	Fig.~\ref{fig:count_time}  shows the performance of the algorithms under $I=100$ and $R=20$, where the solid lines correspond to the average MSEs and dashed lines are for individual trials. One can see that \texttt{SmartCPD} improves the MSE quickly in all trials. On average, it brings the MSE below $10^{-2}$ using less than 10 seconds, while the best baseline, i.e., \texttt{GCP-OPT} takes at least around 40 seconds to reach the same MSE level. 
	Fig.~\ref{fig:count_samp} shows the MSEs of the algorithms against the number of sampled data. Clearly, \texttt{SmartCPD} enjoys the lowest sample complexity in the case under test. In many trials, it uses at least one order of magnitude fewer data entries to reach MSE$=10^{-2}$.
	Fig.~\ref{fig:count_time_100} shows a closer look at the runtime breakouts of the two stochastic algorithms. It can be found that both use similar time for sampling but \texttt{GCP-OPT} requires slightly more time for computing the gradient as well as updating all latent matrices.

	\begin{figure}[ht]
		\centering
		\subfigure[MSE over time.]{
			\includegraphics[width=0.48\linewidth]{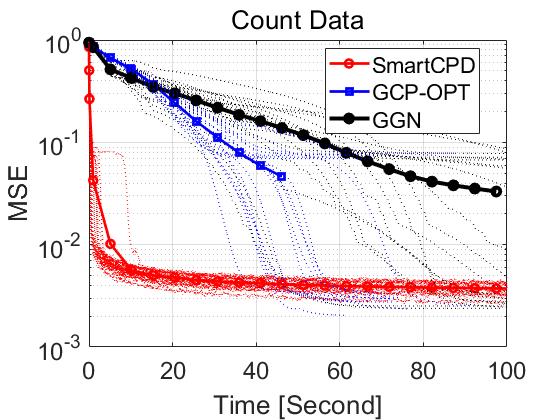}\label{fig:count_time}}
		\subfigure[MSE over no. of samples.]{
			\includegraphics[width=0.48\linewidth]{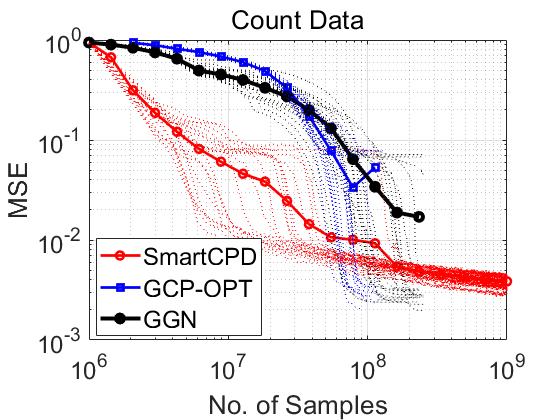}\label{fig:count_samp}}
		\caption{MSE of count tensor ($I=100$ and $R=20$)}
	\end{figure}


	In Table.~\ref{tab:simul_contMSE}, the achieved MSEs after $100$ seconds when $I=100$ and  $R=10,20,50$ are compared. \texttt{SmartCPD} outputs MSEs that are smaller than $10^{-2}$ for all cases, while \texttt{GGN} and \texttt{GCP-OPT} could not reach MSE$\leq 0.1$ in 100 seconds when $R=50$, in terms of both mean and median.

	\begin{figure}[ht]
		\centering
		\includegraphics[width=0.9\linewidth]{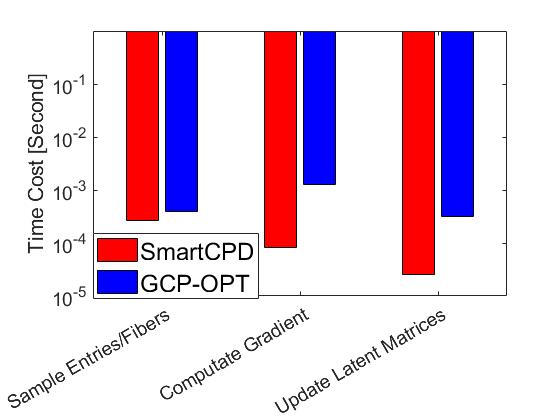}
		\caption{Detailed time cost per iteration of \texttt{SmartCPD} and \texttt{GCP-OPT} algorithms ($I=100$ and $R=20$).}
		\label{fig:count_time_100}
	\end{figure}

	
	\begin{table}[ht]
		\centering
		\small
		\caption{MSE after $100$ seconds ($I=100$ and diff. $R$).}
		\begin{tabular}{c c c c c c}
			\toprule
			Rank $R$                          &   &  $10$&  $20$&  $50$\\ \toprule
			\multirow{2}{*}{\textbf{SmartCPD}}         &  mean  &  2.7E-3&  3.7E-3&  6.6E-3 \\ 
			&  median&  2.7E-3&  3.7E-3&  6.5E-3 \\ 
			\hline
			\multirow{2}{*}{\textbf{GCP-OPT}}           &  mean  &  9.6E-3&  0.028&  0.408\\ 
			&  median&  2.3E-3&  3.3E-3&  0.406  \\ 
			\hline
			\multirow{2}{*}{\textbf{GGN}}    &  mean  &  0.023&    0.031&   0.288\\ 
			&  median&  1.9E-3&    6.5E-3&   0.281\\ 
			\bottomrule
		\end{tabular}
		\label{tab:simul_contMSE}
	\end{table}

	\subsubsection{Binary Data}\label{subsec:simul_binary}
	Next, we evaluate the performance on binary tensors. 
	{We use the loss function that is related to the MLE of Bernoulli tensors, i.e., $\ell(x,m)=\log(m+1)-x\log(m+\epsilon)$ (cf. Table~\ref{tab:ell}).
		Note that \texttt{GGN} is developed for $\beta$-divergence, and thus cannot be used for this loss function.  Hence, we only use \texttt{GCP-OPT} to benchmark our algorithm in the binary case. We set $I=100$ and $R=20$ as before. Each entry of the binary tensor is generated from the Bernoulli distribution, i.e., $\underline{\bm{X}}_{\bm{i}}=1$ with probability ${\underline{\bm{M}}_{\bm{i}}}/{(1+\underline{\bm{M}}_{\bm{i}})}$. The latent matrices are generated as in Sec.~\ref{subsec:simul_count}. We set the constant $A_{\textrm{max}}\in\{0.2,0.3,0.4,0.5\}$ to generate four different binary data tensors, with about $5\%,15\%,28\%,40\%$ nonzero entries, respectively. The sets $\mathcal{\bm{A}}_n$ for all $n$ are set to be the nonnegative orthant.}

	Fig.~\ref{fig:binaryall} shows the MSEs against time of $20$ independent trials when the tensor has $15\%$ nonzero entries.
	Similar as before, \texttt{SmartCPD} requires much shorter time to achieve MSE$\leq 10^{-2}$. In addition, the histograms of MSEs after $60$ seconds are presented in Fig.~\ref{fig:binary_hist} in the supplementary material. One can see from there that that \texttt{SmartCPD} consistently outperforms the baseline, and the advantage is more articulated when the data becomes denser.

	\begin{figure}[ht]
		\centering
		{
			\includegraphics[width=0.9\linewidth]{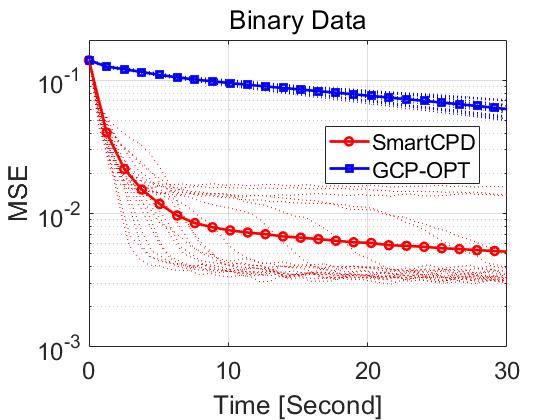}}
		\caption{MSE of binary tensor ($I=100$ and $R=20$).}
		\label{fig:binaryall}
	\end{figure}

	\subsubsection{Continuous Data}\label{subsec:simul_continuous}
	We also evaluate the performance of \texttt{SmartCPD} on continuous tensors under $\beta$-divergence. We consider the multiplicative Gamma noise, i.e., 
	$$\texttt{Gamma:}\ \underline{\bm{X}}_{\bm{i}}=\underline{\bm{M}}_{\bm{i}}\cdot\underline{\bm{N}}_{\bm{i}},$$ 
	in which
	$\underline{\bm{N}}_{\bm{i}}$ is i.i.d. Gamma noise. 
	The SNR for Gamma noise is defined the same as in~\cite{vandecappelle2020second}.
	The latent matrices $\bm{A}_1,\bm{A}_2$, and $\bm{A}_3$ are drawn from the i.i.d. uniform distribution between $0$ and $1$ and $\mathcal{\bm{A}}_n=\mathbb{R}^{I\times R}_+$. Since the $\beta$-divergence loss functions satisfy the convex-concave property, function $\phi$ and $\bm{\Gamma}_n^t$ is used based on the Jensen's inequality for \texttt{SmartCPD}; see details in Table~\ref{tab:jensen} of the supplementary material. 
	We find such setup particularly efficient for dealing with dense and continuous tensors under $\beta$-divergence. 
	
	
	Fig.~\ref{fig:beta_gamma} shows simulations where
	the tensor has a size of $300\times300\times300$ and $R=20$. The average MSEs over $20$ independent trials for $\beta=0$ with SNR$=20$dB. Clearly, the two stochastic algorithms, i.e., \texttt{SmartCPD} and \texttt{GCP-OPT} have faster convergence than \texttt{GGN}. Further, \texttt{SmartCPD} is even faster than \texttt{GCP-OPT}, especially in the beginning. We also observe that the original \texttt{SmartCPD}'s MSE saturates at a certain level after few iterations. One way to improve the MSE is changing the step size scheme from Jensen's inequality scheme to the adaptive step size scheme in \eqref{eq:szL} after some iterations. In Fig.~\ref{fig:beta_gamma}, \texttt{SmartCPD-Mixed} refers to such a step-size scheme, where the step size scheme is changed to that in \eqref{eq:szL} with $\phi(a)=\frac{1}{2}a^2$ when the difference of the loss function between two consecutive epochs\footnote{\scriptsize An epoch means a sweep of $|\mathcal{I}|$ tensor entries.} is less than $10^{-4}$. It can be observed that \texttt{SmartCPD-Mixed} offers the best runtime and MSE performance.

	\begin{figure}[ht]
		\centering
		\includegraphics[width=0.95\linewidth]{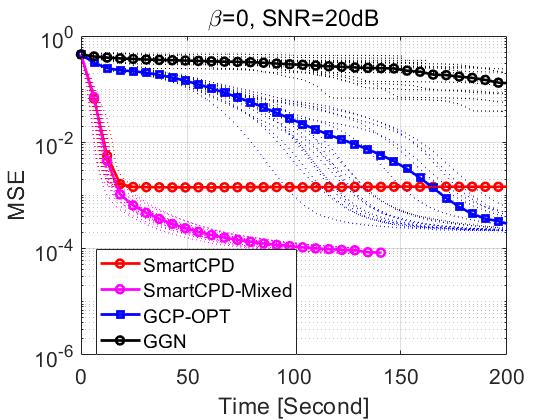}
		\caption{MSE of continuous data tensor with multiplicative Gamma noise ($I_n=300$, $R=20$, and $\beta=0$).}
		\label{fig:beta_gamma}
	\end{figure}

	\subsubsection{Column Constraints}
	Finally, we evaluate the performance for problems with simplex constraints, i.e., $\bm{A}_n^T\bm{1}=\bm{1},\bm{A}_n(i,j)\geq 0,\forall i,j$.  Tensors that represent joint probability distributions with $I=100$ and $R=10$ is considered; see \cite{ibrahim2020recovering,yeredor2019estimation,yeredor2019maximum,kargas2019learning,kargas2018tensors} for details. The latent matrices $\bm{A}_1,\bm{A}_2$, and $\bm{A}_3$ are firstly drawn from i.i.d. uniform distribution between $0$ and $1$, then each column is normalized so that it represents a probability mass function. The generalized KL divergence is used as loss function. For \texttt{SmartCPD}, the step size scheme in \eqref{eq:szL} with $\phi(a)=a\log a$ is used. 
	Note that both the baselines are not able to handle such simplex constraint, while this problem frequently arise in probabilistic tensor decomposition; see \cite{yeredor2019estimation,yeredor2019maximum,ibrahim2019crowdsourcing,kargas2019learning}.
	To benchmark our algorithm, we use the deterministic block MD algorithm in~\cite{kargas2019learning}.
	In Fig.~\ref{fig:simplex}, one can see that \texttt{SmartCPD} has much fast convergence behavior than the MD algorithm under the probabilistic simplex constraint.
	
	\begin{figure}[ht]
		\centering
		\includegraphics[width=0.95\linewidth]{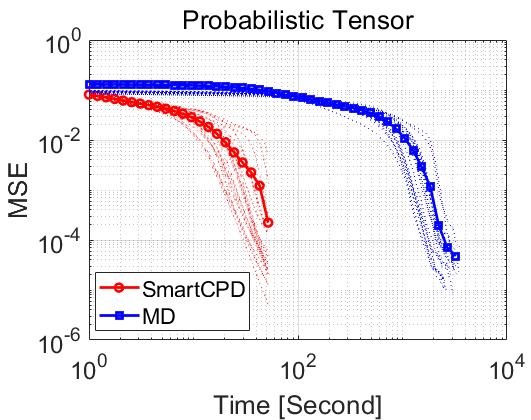}
		\caption{MSE of continuous data tensor with simplex constraint ($I=100$, $R=10$, and $\beta=1$).}
		\label{fig:simplex}
	\end{figure}


	\subsection{Real Data}
	\subsubsection{Chicago Crime Data}
	{We apply the algorithms onto the} Chicago crime dataset. The dataset records crime reports in the city of Chicago, Illinois, United States, between January 1, 2001 to December 11, 2017. The original dataset is published in the official website of the city of Chicago (\url{www.cityofchicago.org}). Here, we use the version in \cite{frosttdataset}. The data is in the form of a fourth order tensor (day$\times$hour$\times$hour$\times$community) with integer entries representing the number of crimes reported. The size of the tensor is $6186 \times 24 \times 77 \times 32$. It has 5,330,678 ($\approx 1.5\%$) nonzero entries. 
	
	We choose the loss function corresponding to the Poisson distribution, i.e., $\ell(x,m)=m-x\log(m+\epsilon)$ (see Table~\ref{tab:ell}). 
	In every iteration, $40$ fibers are used by \texttt{SmartCPD}. For \texttt{GCP-OPT}, an equal amount of data, i.e., $40 \times 6186$ tensor entries, are sampled in every iteration. 
	All algorithms under test are stopped when the relative change in the $\beta$-divergence (in this case, $\beta=1$) is less than $10^{-3}$.  

	Fig.~\ref{fig:conv_chicago} shows the cost change against time of the algorithms under $R=5$ and $R=10$, respectively. Each algorithm is run for 20 trials and in each trial, the factor matrices are initialized by randomly sampling its entries from uniform distribution between 0 and 1. One can see that the proposed algorithm \texttt{SmartCPD} exhibits a fast runtime performance in this case. For each of the cases, the \texttt{SmartCPD} takes only about 5 seconds to reach a low cost value, whereas the baselines take much more time but still could not attain the same cost value. In some trials, especially when $R=5$, there exist some cases where \texttt{GGN} did not converge. Some more details on the algorithm-output latent factors can be found in the supplementary material.
	
	
	\begin{figure}[t]
		\centering
		\subfigure[$R=5$]{
			\includegraphics[width=0.48\linewidth]{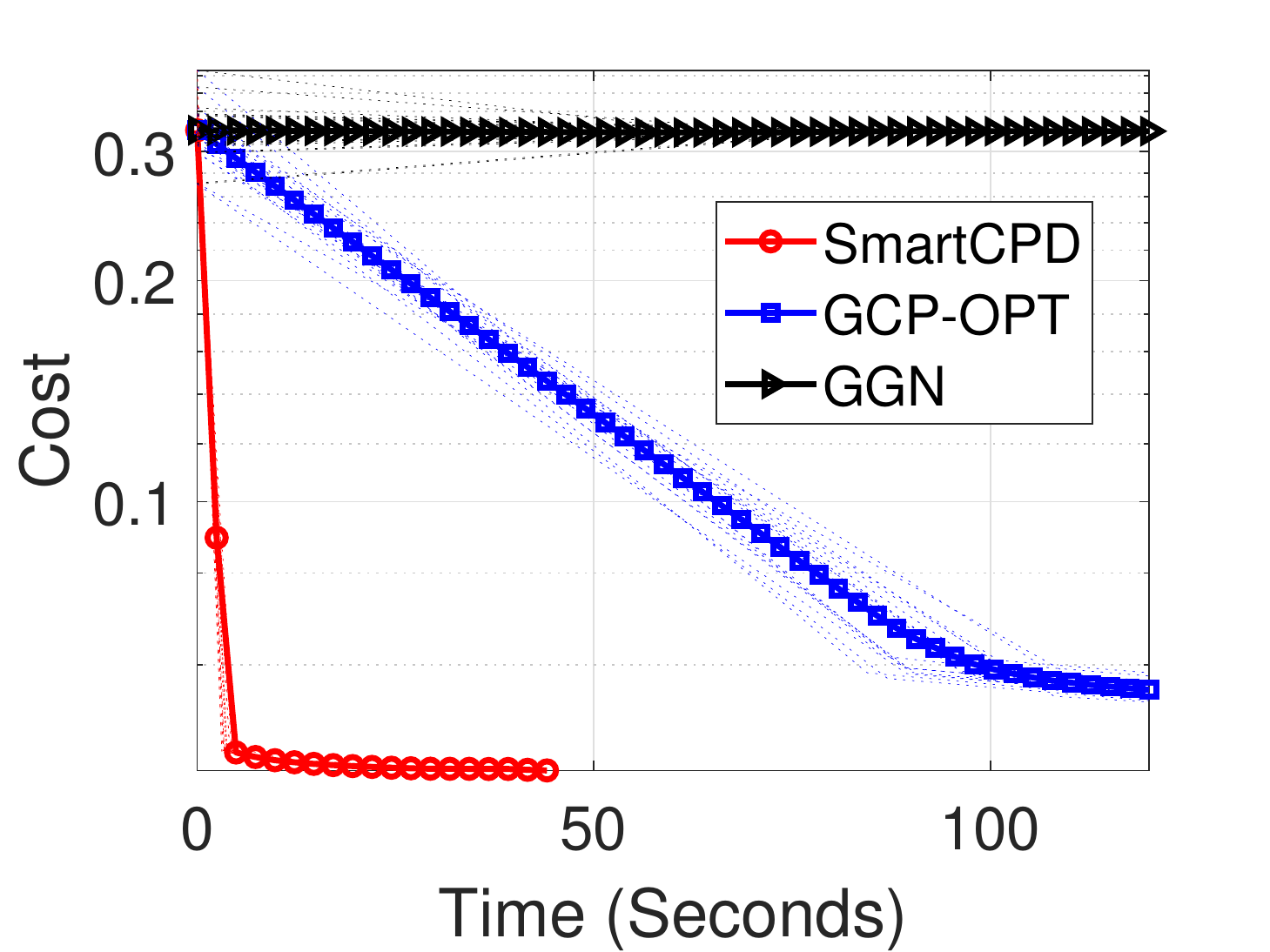}}
		\subfigure[$R=10$]{
			\includegraphics[width=0.48\linewidth]{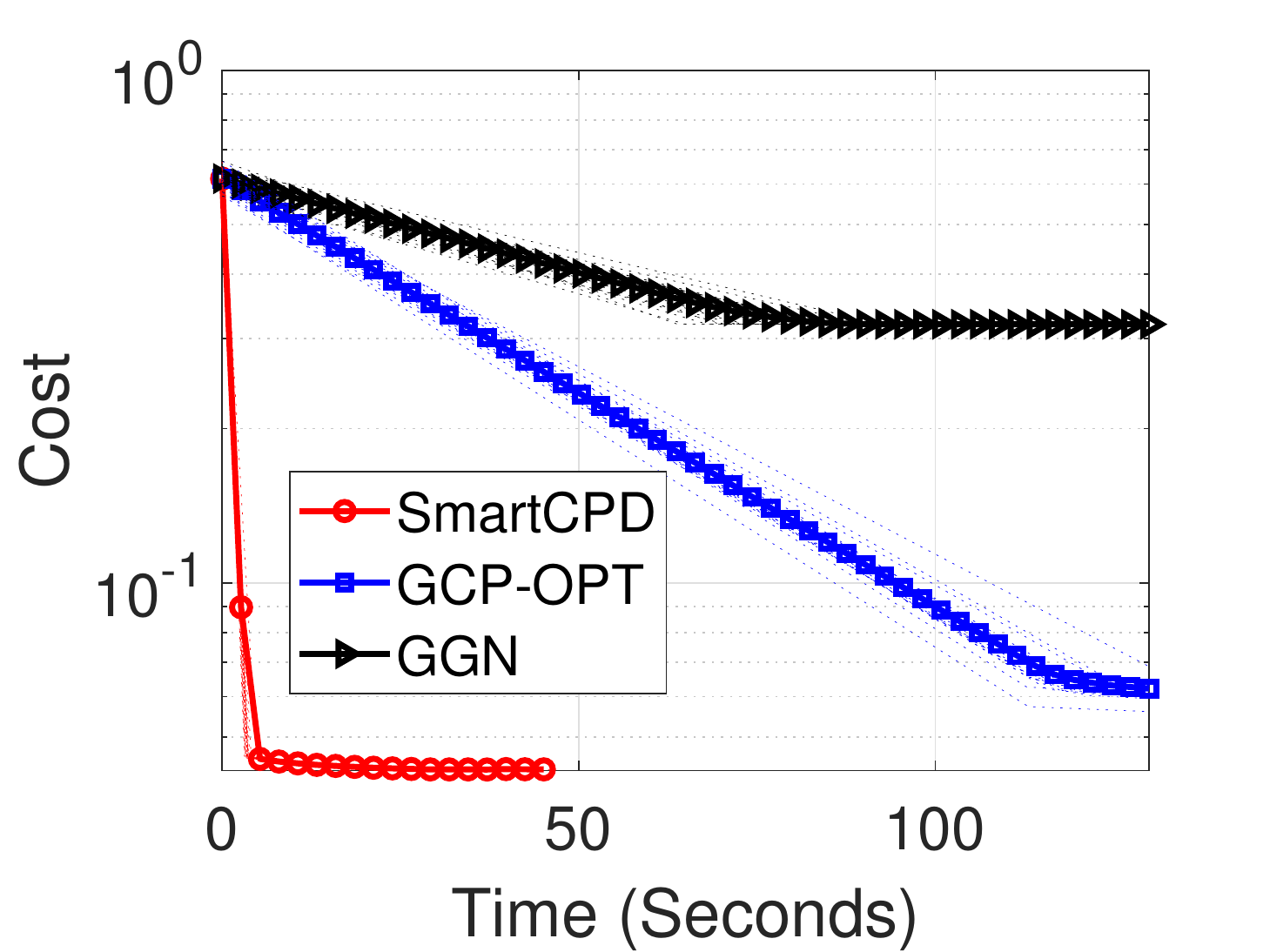}}
		\caption{Convergence of the algorithms (integer data, Chicago crime, KL div., size $=6186 \times 24 \times 77 \times 32$).}
		\label{fig:conv_chicago}
	\end{figure}



	\subsubsection{Plant-Pollinator Network Data}
	We also consider the plant-pollinator network dataset published by \cite{jones2020plant}. The dataset consists of plant-pollinator interactions collected over 12 meadows in the Oregon Cascade Mountains, USA; {also see \cite{Fu2019Link} for detailed data descriptions}. We extract the number of interactions between 562 pollinator species and 124 plant species over 123 days during 2011 and 2015. Hence, we form a count-type 
	third-order tensor {with a} size of $562 \times 124 \times 123 $, where each entry represents the number of interactions between a particular plant and a pollinator on a specific day. The data is {fairly} sparse with only 8,370 ($\approx 0.1\%$) nonzero entries. We choose the same algorithm settings as used in the Chicago Crime data.
	
	\begin{figure}[t]
		\centering
		\subfigure[$R=10$]{
			\includegraphics[width=0.48\linewidth]{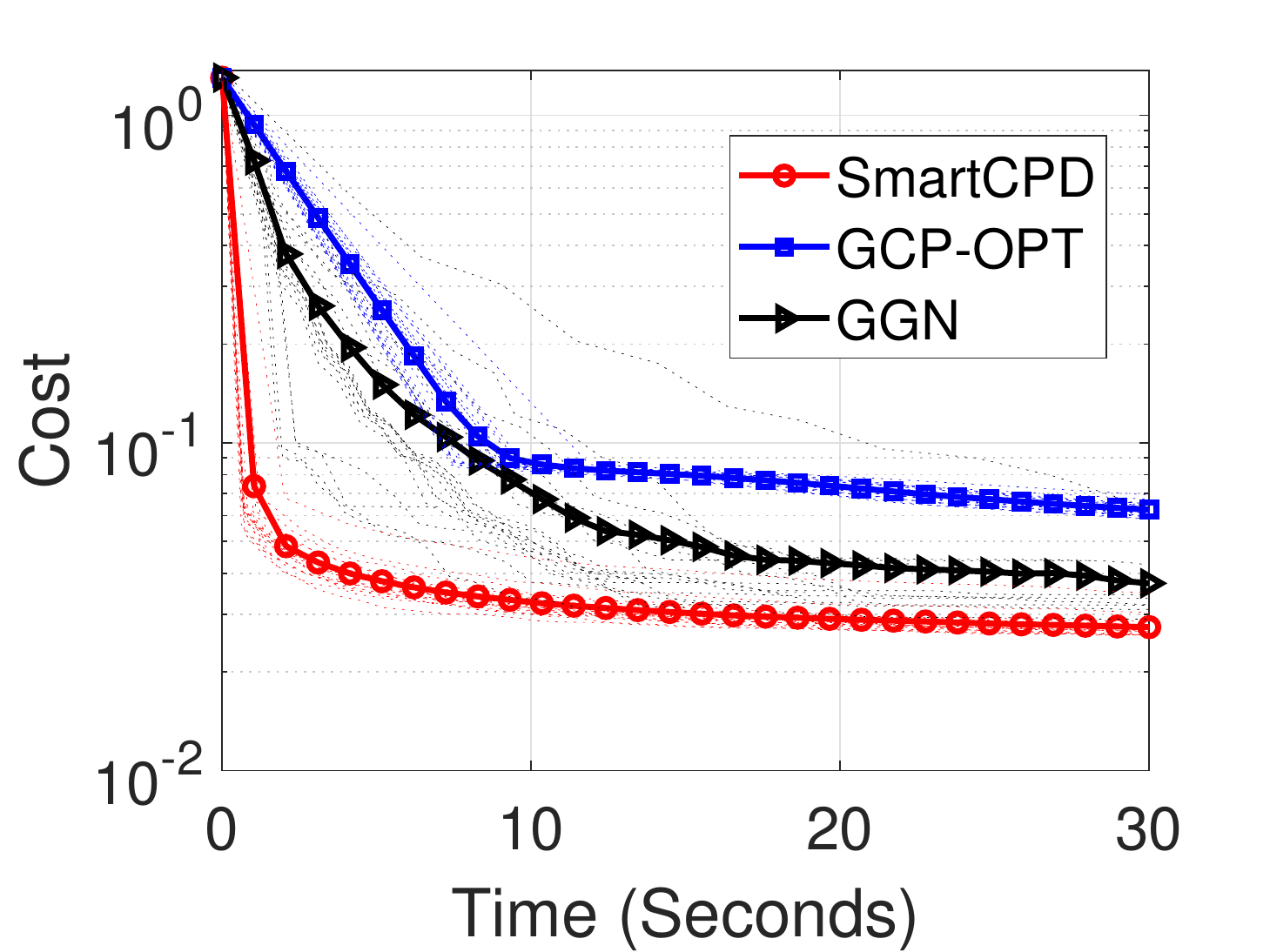}}
		\subfigure[$R=20$]{
			\includegraphics[width=0.48\linewidth]{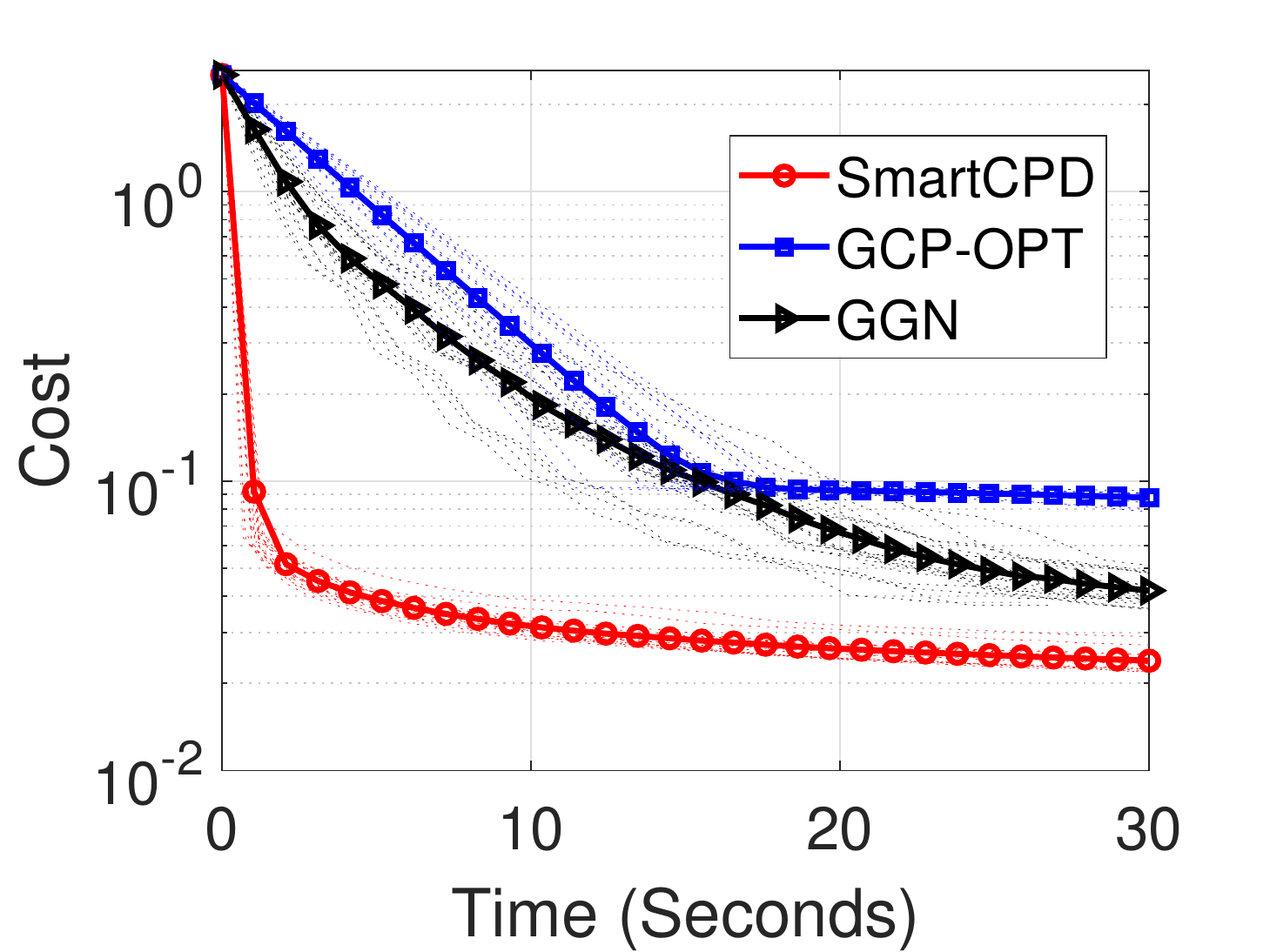}}
		\caption{Convergence of the algorithms (integer data, Plant-Pollinator network, gen. KL div., size $=562 \times 124 \times 123$).}
		\label{fig:Plant}
	\end{figure}
	
	Fig.~\ref{fig:Plant} plots the cost values of the algorithms against time, for different values of $R$. Both the baselines \texttt{GCP-OPT} and \texttt{GGN} work reasonably well in this dataset. However, the proposed \texttt{SmartCPD} has better performance---it converges to lower cost values compared to the other baselines. Also, one can observe that the \texttt{SmartCPD} is about 15 times faster in reaching low cost values compared to \texttt{GCP-OPT} and \texttt{GGN}. 
	
	
	
	\subsubsection{UCI Chat Network Data}
	We test the algorithms on real-world binary data using a social network which contains the online interactions of the students from the University of California, Irvine, USA. The original dataset was published by \cite{opsahl2009clustering}, which includes 59,835 online messages sent between 1,899 students over 196 days from March 2004 to October 2004. We select 400 most prolific senders and form a third-order binary tensor of size $400 \times 400 \times 196$ having 18862 ($\approx 0.06\%$) nonzero entries. Each entry of the binary tensor indicates if sender $i$ has sent a message to receiver $j$ on the $k$th day. We choose the loss function corresponding to the Bernoulli distribution, i.e., $\ell(m,x)=\log(m+1)-x\log(m+\epsilon)$ (see Table~\ref{tab:ell}). Other settings and parameters are as before. Since \texttt{GGN} method is not designed for the loss function considered in this case, it is not included.
	
	\begin{figure}[t]
		\centering
		\subfigure[$R=10$]{
			\includegraphics[width=0.48\linewidth]{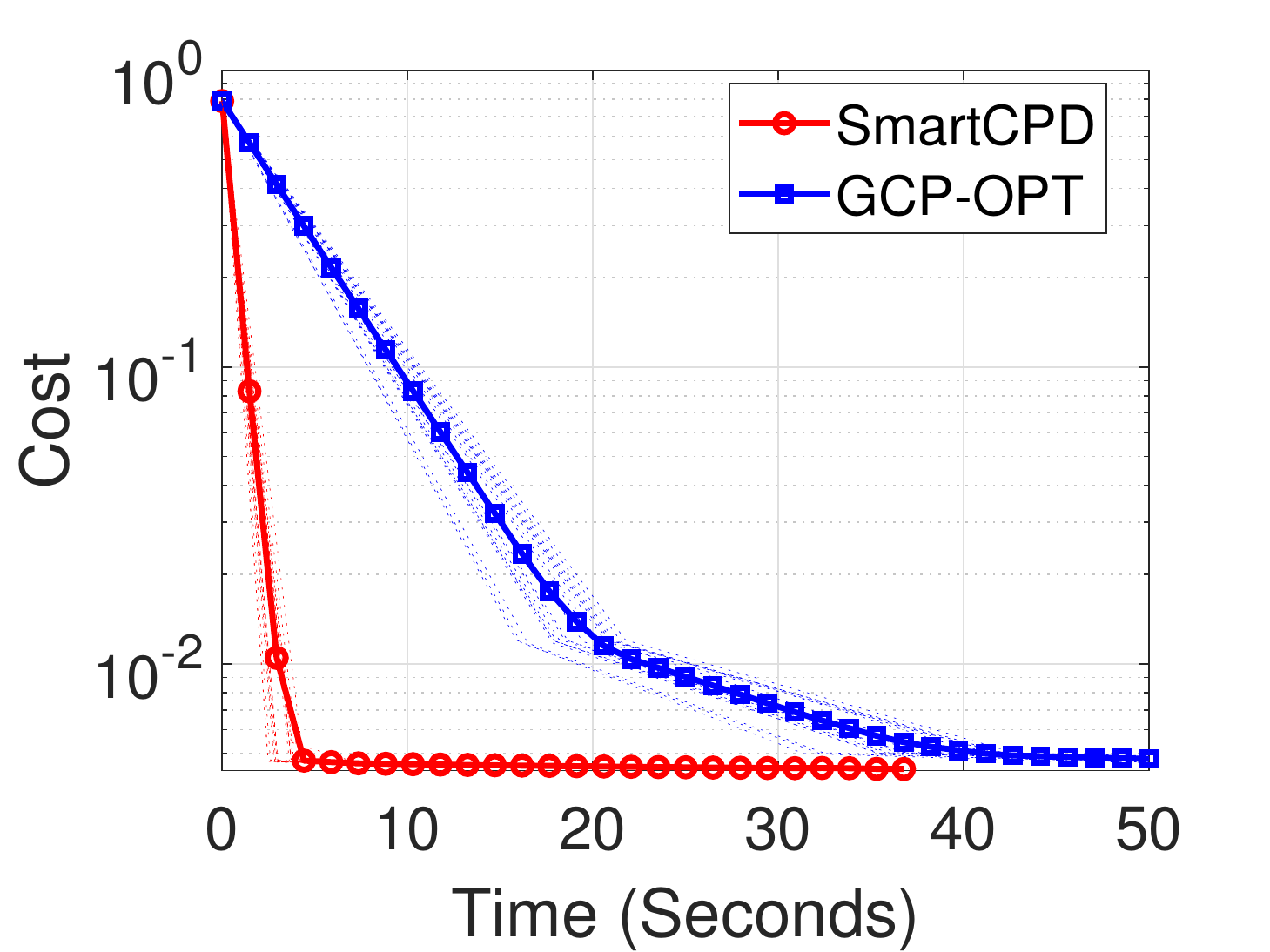}}
		\subfigure[$R=20$]{
			\includegraphics[width=0.48\linewidth]{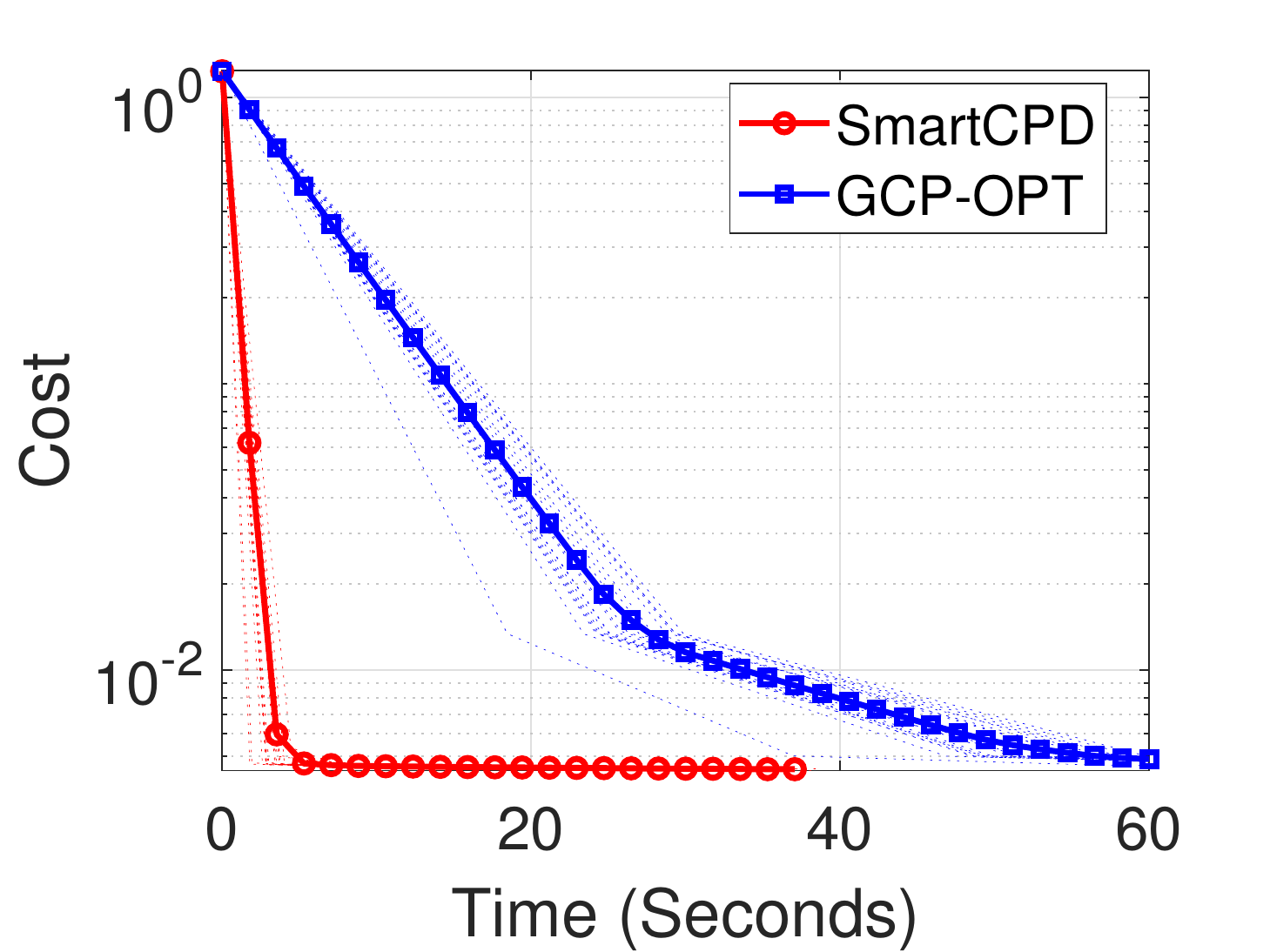}}
		\caption{Convergence of the algorithms (binary data, UCI chat network, log loss, size $=400 \times 400 \times 196$).}
		\label{fig:UC}
	\end{figure}
	
	Fig.~\ref{fig:UC} shows the cost value change against time in seconds for different values of $R$. Similar to the previous datasets, the proposed \texttt{SmartCPD} shows considerable runtime advantages over \texttt{GCP-OPT}. In all the trials, \texttt{SmartCPD} is at least 40 times faster for converging to a cost value that is later attained by \texttt{GCP-OPT}.

	
	
	
	\section{Conclusion}
	\label{sec:conl}
	In this work, we proposed a unified SMD algorithmic framework for low-rank CPD under non-Euclidean losses. By integrating a fiber-sampling strategy within the SMD optimization technique, the proposed framework is flexible in dealing with a variety of loss functions and constraints that are of interest in real-world data analytics. By its stochastic nature, the proposed algorithm enjoys low computational and memory costs. 
	In addition, under different data types and loss functions, we discussed a number of ``best practices'', e.g., step size scheduling and local surrogate function construction, which were shown critical for effective implementation.
	We also provided rigorous convergence analysis that is tailored for the non-Euclidean CPD, since generic SMD proofs do not cover the proposed algorithm.
	We tested the algorithm over various types of simulated and real data.
	Substantial computational savings relative to state-of-the-art methods were observed. 
	These results show encouraging and promising performance of using geometry-aware algorithm design for large-scale tensor decomposition.
	


	\footnotesize
	\bibliographystyle{IEEEbib}
	\bibliography{ref}
	
	\vfill\pagebreak
	\newpage
	\clearpage
	\begin{center} 
		{\large\textbf{Supplementary Materials}}
	\end{center}
	\normalsize
	\appendices
	\section{Proof of Lemma~\ref{lem:phi}}\label{app:phi}
	Decompose $\ell(x,\bm{h}^T\bm{a})$ into a convex part $\check{\ell}(x,\bm{h}^T\bm{a})$ plus a concave part $\hat{\ell}(x,\bm{h}^T\bm{a})$ (if exists), we have 
	\begin{equation*}
		\begin{aligned}
			\ell(x,\bm{h}^T\bm{a}) = \check{\ell}(x,\bm{h}^T\bm{a})+\hat{\ell}(x,\bm{h}^T\bm{a}).
		\end{aligned}
	\end{equation*}
	For convex part $\check{\ell}(x,\bm{h}^T\bm{a})$ we have
	\begin{equation}\label{eq:js}
		\begin{aligned}
			\check{\ell}(x,\bm{h}^T\bm{a})&=\check{\ell}(x,\sum_{r}\lambda_r \frac{h_r}{\lambda_r}a_r)\\
			&\mathop{\leq}\limits^{(i)}\sum_{r}\lambda_r \check{\ell}(x,\frac{h_r}{\lambda_r}a_r),
		\end{aligned}
	\end{equation}
	where $(i)$ is due the Jensen's inequality. Since $L\phi(a_r)-\lambda_r \check{\ell}(x,\frac{h_r}{\lambda_r}a_r)$ is convex, then by the convexity we have
	\begin{equation}\label{eq:lcvx_up}
		\begin{aligned}
			&L\phi(\bar{a}_r)-\lambda_r \check{\ell}(x,\frac{h_r}{\lambda_r}\bar{a}_r)\\
			&\quad + \langle L\nabla \phi(\bar{a}_r)-\lambda_r \nabla \check{\ell}(x,\frac{h_r}{\lambda_r}\bar{a}_r), a_r - \bar{a}_r \rangle \\
			&\leq L\phi(a_r)-\lambda_r \check{\ell}(x,\frac{h_r}{\lambda_r}a_r)
		\end{aligned}
	\end{equation}
	where $\nabla \check{\ell}(x,\frac{h_r}{\lambda_r}\bar{a}_r)$ represents the derivative of $\check{\ell}(x,\frac{h_r}{\lambda_r}\bar{a}_r)$ with respect to $a_r$.
	Substitute~\eqref{eq:lcvx_up} into \eqref{eq:js}, we have
	\begin{equation}\label{eq:lcvx}
		\begin{aligned}
			&\check{\ell}(x,\bm{h}^Ta)\\
			&\leq \sum_{r} \lambda_r \check{\ell}(x,\frac{h_r}{\lambda_r}\bar{a}_r)+\langle \lambda_r \nabla \check{\ell}(x,\frac{h_r}{\lambda_r}\bar{a}_r), a_r - \bar{a}_r \rangle \\
			&\quad + L\phi(a_r)-L\phi(\bar{a}_r)-\langle L\nabla \phi(\bar{a}_r), a_r - \bar{a}_r \rangle\\
			&= \sum_{r} \lambda_r \check{\ell}(x,\frac{h_r}{\lambda_r}\bar{a}_r)+\langle \lambda_r \nabla \check{\ell}(x,\frac{h_r}{\lambda_r}\bar{a}_r), a_r - \bar{a}_r \rangle \\
			&\quad+ L D_\phi(a_r,\bar{a}_r)\\
			&\mathop{=}\limits^{(i)} \check{\ell}(x,\bm{h}^T\bar{\bm{a}})+\sum_{r}\langle \nabla_r \check{\ell}(x,\bm{h}^T\bar{a}), a_r - \bar{a}_r \rangle +  L D_\phi(a_r, \bar{a}_r),\\
			&=\check{\ell}(x,\bm{h}^T\bar{\bm{a}})+\langle \nabla \check{\ell}(x,\bm{h}^T\bar{\bm{a}}), \bm{a} - \bar{\bm{a}} \rangle+LD_{\phi}(\bm{a},\bar{\bm{a}}),
		\end{aligned}
	\end{equation}
	where $\nabla_r \check{\ell}(x,\bm{h}^T\bar{\bm{a}})$ denotes the $r$th compoent of the gradient of $\check{\ell}(x,\bm{h}^T\bar{\bm{a}})$ with respect to $a$, inequality $(i)$ is by the definition of $\lambda_r$ and $\lambda_r \nabla \check{\ell}(x,\frac{h_r}{\lambda_r}\bar{a}_r)=\nabla_r \check{\ell}(x,h^T\bar{\bm{a}})$.
	As for the concave part $\hat{\ell}(x,\bm{h}^Ta)$, by the concavity we have
	\begin{equation}\label{eq:lccv}
		\begin{aligned}
			\hat{\ell}(x,\bm{h}^T\bm{a})\leq \hat{\ell}(x,\bm{h}^T\bar{\bm{a}})+\langle \nabla \hat{\ell}(x,\bm{h}^T\bar{\bm{a}}), \bm{a} - \bar{\bm{a}} \rangle.
		\end{aligned}
	\end{equation}
	Summing up inequalities \eqref{eq:lcvx} and \eqref{eq:lccv} completes the proof.
	
	\section{Proof of Lemma~\ref{lem:L}}\label{app:L}
	For notation simplicity, denote $x_n=\textrm{vec}(\bm{A}_n)$ and the constraint set $\mathcal{\bm{A}}_n$ is re-introduced correspondingly as $\mathcal{X}_n$, i.e., $x_n\in\mathcal{X}_n$. Consequentially the objective function in \eqref{eq:optori} is denoted as $F(\bm{x})$ with $\bm{x}=(x_1,x_2,\ldots,x_N)$. 
	Further, define the entry-wise correspondence between $x_n$ and $\bm{A}_n$ as $x_n^{i,r}=\bm{A}_n(i,r)$ and use $x_n^{i,r}$ to  represent the $[(i-1)I_n+r]$th component of vector $x_n$. Then, $F(\bm{x})$ in~\eqref{eq:optref} can be re-expressed as
	\begin{equation*}
		\begin{aligned}
			F(\bm{x})=\sum_{\bm{i}}\ell (\mathcal{T}_{\bm{i}},\sum_{r=1}^R\prod_{n=1}^N x_n^{i_n,r})=\sum_{\bm{i}}\ell (\mathcal{T}_{\bm{i}},g_{\bm{i}}(\bm{x})),
		\end{aligned}
	\end{equation*}
	where $\bm{i}=(i_1,i_2,\ldots,i_N)\in\mathcal{I}$ is the tensor entry index as introduced in Section~\ref{sec:intro} and $g_{\bm{i}}(\bm{x})$ is a specific $R$th-order polynomial which only depends on few components of $\bm{x}$. Since $F(\bm{x})$ is in a finite-sum form, we first consider its the Hessian matrix for a fixed $\bm{i}$. Denote $\nabla^2 \ell(\mathcal{T}_{\bm{i}},g_{\bm{i}}(\bm{x}))$ as the Hessian matrix with respect to $\bm{x}$, then by the chain rule we have
	\begin{equation}\label{eq:hessian_l}
		\begin{aligned}
			\nabla^2 \ell(\mathcal{T}_{\bm{i}},g_{\bm{i}}(\bm{x}))&=\ell^{\prime \prime}(\mathcal{T}_{\bm{i}},g_{\bm{i}}(\bm{x}))\left(\nabla  g_{\bm{i}}(\bm{x})\nabla g_{\bm{i}}(\bm{x})^T\right)\\
			&\quad+\ell^{\prime}(\mathcal{T}_{\bm{i}},g_{\bm{i}}(\bm{x}))\nabla^2 g_{\bm{i}}(\bm{x})),
		\end{aligned}
	\end{equation}
	where $\ell^{\prime}(\mathcal{T}_{\bm{i}},y)$ and $\ell^{\prime \prime}(\mathcal{T}_{\bm{i}},y)$ denote the first- and second-order derivatives of $\ell(\mathcal{T}_{\bm{i}},y)$ with respect to $y$ respectively, $\mathcal{T}_{\bm{i}}=\underline{\bm{X}}_{\bm{i}}$, $\nabla  g_{\bm{i}}(\bm{x})$ is the gradient of $g_{\bm{i}}(
	\cdot)$ and $\nabla^2 g_{\bm{i}}(\bm{x})$ is its Hessian matrix. By the compactness of the constraint set $\mathcal{X}$, we have the fact (see Fact~\ref{fact:JH}) that there are constants $J_{\bm{i}}<\infty$ and $H_{\bm{i}}<\infty$ such that 
	$$J_{\bm{i}} I \pm g_{\bm{i}}(\bm{x})\nabla g_{\bm{i}}(\bm{x})^T \succeq \bm{0},\ H_{\bm{i}} I \pm \nabla^2 g_{\bm{i}}(\bm{x})\succeq \bm{0}$$ 
	Also, by Fact~\ref{fact:l}, we have  $$|\ell^{\prime}(\mathcal{T}_{\bm{i}},g_{\bm{i}}(\bm{x}))|\leq U^\prime_{\bm{i}},\ |\ell^{\prime\prime}(\mathcal{T}_{\bm{i}},g_{\bm{i}}(\bm{x}))|\leq  U^{\prime\prime}_{\bm{i}},\forall \bm{x}\in\mathcal{X},$$
	where $U^\prime_{\bm{i}}<\infty$ and $U^{\prime\prime}_{\bm{i}}<\infty$ are some constants. Hence, combine~\eqref{eq:hessian_l} with $\sigma$-strong convexity of function $\phi(\bm{x})$ over the compact set $\mathcal{X}$, we have 
	\begin{equation}
		\frac{U^\prime_{\bm{i}} J_{\bm{i}}+U^{\prime\prime}_{\bm{i}} H_{\bm{i}}}{\sigma} \nabla^2 \phi(\bm{x},\bm{x})\pm \nabla^2 \ell(\mathcal{T}_{\bm{i}},g_{\bm{i}}(\bm{x}))  \succeq \bm{0}.
	\end{equation}
	Set $L=\sum_{\bm{i}}\frac{U^\prime_{\bm{i}} J_{\bm{i}}+U^{\prime\prime}_{\bm{i}} H_{\bm{i}}}{\sigma}<\infty$, we obtain 
	$L\phi(\bm{x})\pm \nabla^2 F(\bm{x})\succeq \bm{0}$, which implies $\phi(\bm{x})\pm F(\bm{x})$ is convex. By the convexity of functions $L\phi(\bm{x})\pm F(\bm{x})$, we complete the proof.
	
	\begin{fact}\label{fact:JH}
		For the compact set $\mathcal{X}$, there exist constants $J_{\bm{i}}$ and $H_{\bm{i}}$ such that 
		\begin{align*}
			&\max\{ |\lambda_{\textrm{min}}(\nabla g_{\bm{i}}(\bm{x})\nabla g_{\bm{i}}(\bm{x})^T)|,|\lambda_{\textrm{max}}(g_{\bm{i}}(\bm{x})\nabla g_{\bm{i}}(\bm{x})^T)|\}\leq J_{\bm{i}}, \\
			&\max\{|\lambda_{\textrm{min}}(\nabla^2 g_{\bm{i}}(\bm{x}))|,|\lambda_{\textrm{max}}(\nabla^2 g_{\bm{i}}(\bm{x}))|\} \leq H_{\bm{i}}.
		\end{align*}
		Further, for symmetric matrices $g_{\bm{i}}(\bm{x})\nabla g_{\bm{i}}(\bm{x})^T$ and $\nabla^2 g_{\bm{i}}(\bm{x})$, we immediately have  
		$$J_{\bm{i}}I\pm \nabla g_{\bm{i}}(\bm{x})\nabla g_{\bm{i}}(\bm{x})^T\succeq \bm{0},\ H_{\bm{i}} I \pm \nabla^2 g_{\bm{i}}(\bm{x})\succeq \bm{0}.$$
	\end{fact}
	
	\begin{proof}
		Since $g_{\bm{i}}(\bm{x})$ only depends on few components of $\bm{x}$, both $\nabla  g_{\bm{i}}(\bm{x})$ and $\nabla^2  g_{\bm{i}}(\bm{x})$ contain many zero entries. To clearly understand the inside structure of non-zero entries, with abuse of notation, we ignore index $\bm{i}$ and define $x_n^{i_n,r}=x_n^{r}$ and use $g(\bm{x}_1,\bm{x}_2,\ldots,\bm{x}_N)$ to denote $g_{\bm{i}}(\bm{x})$, where $\bm{x}_n=(x_n^{1},x_n^{2},\ldots,x_n^{R})$. With such notations, we know that $\nabla  g(\bm{x}_1,\bm{x}_2,\ldots,\bm{x}_N)$ and $\nabla^2 g(\bm{x}_1,\bm{x}_2,\ldots,\bm{x}_N)$ correspond to non-zero block parts of $\nabla  g_{\bm{i}}(\bm{x})$ and $\nabla^2 g_{\bm{i}}(\bm{x})$ respectively. The gradient of $g(\bm{x}_1,\bm{x}_2,\ldots,\bm{x}_N)$ can be expressed as 
		\begin{equation*}
			\nabla  g(\bm{x}_1,\bm{x}_2,\ldots,\bm{x}_N)^T = [\nabla_1^T,\nabla_n^T,\ldots,\nabla_N^T]\triangleq \nabla g^T,
		\end{equation*}
		where $\nabla_n=[\prod_{i\neq n}x_i^{1},\prod_{i\neq n}x_i^{2},\ldots,\prod_{i\neq n}x_i^{R}]^T$. Similarly, $\nabla^2 g_{\bm{i}}(\bm{x})$ can be expressed as 
		\begin{equation*}
			\nabla^2 g(\bm{x}_1,\bm{x}_2,\ldots,\bm{x}_N) = \begin{bmatrix} 
				\nabla^2_{11} & \nabla^2_{12} & \ldots & \nabla^2_{1N}\\
				\nabla^2_{21} & \nabla^2_{22} & \ldots & \nabla^2_{2N}\\
				\vdots & \vdots &  & \vdots\\
				\nabla^2_{N1} & \nabla^2_{12} & \ldots & \nabla^2_{NN}
			\end{bmatrix}\triangleq \nabla^2 g,
		\end{equation*}
		where 
		$$\nabla^2_{nn^\prime}=\textrm{diag}\left(\left[\prod_{i\neq n,n^\prime}x_i^{1},\prod_{i\neq n,n^\prime}x_i^{2},\ldots,\prod_{i\neq n,n^\prime}x_i^{R}\right]\right),\forall n\neq n^\prime$$
		and $\nabla^2_{nn}=\bm{0},\forall n$. 
		The Frobenius norm of $\nabla  g_{\bm{i}}(\bm{x})\nabla g_{\bm{i}}(\bm{x})^T$ and $\nabla^2 g_{\bm{i}}(\bm{x})$ are equal to that of $\nabla  g\nabla  g^T$ and $\nabla^2 g$ respectively. Since $\bm{x}$ belongs to a compact set $\mathcal{X}$, we have an upper bound for each $|x_n^r|$, i.e., $| x_n^r |\leq C,\forall n,r$. This further implies
		$$
		\| \nabla  g \nabla  g^T \|_F^2\leq R^2N^2C^{2N-2},\ \|\nabla^2  g\|_F^2\leq N^2RC^{N-2}.
		$$
		As $\nabla  g\nabla  g^T$ and $\nabla^2  g$ are symmetric matrices which have real eigenvalues that upper bounded their corresponding Frobenius norm. Define constants $J=RNC^{N-1}$ and $H=N\sqrt{RC^{N-2}}$, we have 
		\begin{align*}
			&\max\{ |\lambda_{\textrm{min}}(\nabla g\nabla g^T)|,|\lambda_{\textrm{max}}(\nabla g\nabla g^T)|\} \leq J,\\
			&\max\{|\lambda_{\textrm{min}}(\nabla^2 g)|,|\lambda_{\textrm{max}}(\nabla^2 g)|\} \leq H.
		\end{align*}
	\end{proof}
	\begin{fact}\label{fact:l}
		For function $\ell(\mathcal{T}_{\bm{i}},g_{\bm{i}}(\bm{x}))$ considered in Table~\ref{tab:ell}, there exist $U_{\bm{i}}^\prime <\infty$ and $U_{\bm{i}}^{\prime\prime} <\infty$ such that 
		$$|\ell^{\prime}(\mathcal{T}_{\bm{i}},g_{\bm{i}}(\bm{x}))|\leq U^\prime_{\bm{i}},\ |\ell^{\prime\prime}(\mathcal{T}_{\bm{i}},g_{\bm{i}}(\bm{x}))|\leq  U^{\prime\prime}_{\bm{i}},\forall \bm{x}\in\mathcal{X}$$
	\end{fact}
	
	\begin{proof}
		The compactness of $\mathcal{X}$ implies the range space of polynomial function $g_{\bm{i}}(\bm{x})$ is lower and upper bounded, i.e., $-\infty<L_{g,\bm{i}}\leq |g_{\bm{i}}(\bm{x})|\leq U_{g,\bm{i}}<\infty$, where $L_{g,\bm{i}}=-RC^N$ and $U_{g,\bm{i}}=RC^N$ and $C$ is the upper bound of each entry of $\bm{x}$ that $|x_n^{i,r}|\leq C$. On the other hand, both $\ell^{\prime}(\mathcal{T}_{\bm{i}},y)$ and $\ell^{\prime \prime}(\mathcal{T}_{\bm{i}},y)$ can be represented as the sum of two monotonic functions, take $\ell^{\prime}(\mathcal{T}_{\bm{i}},y)$ as example, we have 
		\begin{equation*}
			\begin{aligned}
				\ell^{\prime}(\mathcal{T}_{\bm{i}},y)&=\acute{\ell}^{\prime}(\mathcal{T}_{\bm{i}},y)+\grave{\ell}^{\prime}(\mathcal{T}_{\bm{i}},y)
			\end{aligned}
		\end{equation*}
		where $\acute{\ell}^{\prime}(\mathcal{T}_{\bm{i}},y)$ and $\grave{\ell}^{\prime}(\mathcal{T}_{\bm{i}},y)$ are the monotonic increasing and decreasing parts of $\ell^{\prime}$ respectively.
		This implies  $\ell^{\prime}(\mathcal{T}_{\bm{i}},g_{\bm{i}}(\bm{x}))$ is lower and upper bounded over compact set $\mathcal{X}$. Define 
		\begin{align*}
			U^\prime_{\bm{i}}=\max\{& |\acute{\ell}^{\prime}(\mathcal{T}_{\bm{i}},U_{g,\bm{i}})+\grave{\ell}^{\prime}(\mathcal{T}_{\bm{i}},L_{g,\bm{i}})|,\\
			&|\acute{\ell}^{\prime}(\mathcal{T}_{\bm{i}},L_{g,\bm{i}})+\grave{\ell}^{\prime}(\mathcal{T}_{\bm{i}},U_{g,\bm{i}})|\}, 
		\end{align*}
		we immediately have $|\ell^{\prime}(\mathcal{T}_{\bm{i}},g_{\bm{i}}(\bm{x}))|\leq U^\prime_{\bm{i}},\ \forall \bm{x}\in\mathcal{X}$. Similarly, we can find $U^{\prime\prime}_{\bm{i}}$ such that $|\ell^{\prime\prime}(\mathcal{T}_{\bm{i}},g_{\bm{i}}(\bm{x}))|\leq  U^{\prime\prime}_{\bm{i}},\ \forall \bm{x}\in\mathcal{X}$. This completes the proof.
	\end{proof}
	
	\section{Proof of Lemma~\ref{lem:station}}\label{app:sta}
	We use the notation in Appendix~\ref{app:L} for simplicity. We first prove `if' part. By the first-order optimality condition of $\bm{\hat{x}}$, that is, for some $\bm{\hat{v}}\in\partial h(\bm{\hat{x}})$, we have 
	\begin{equation}\label{eq:optcond_sta}
		\nabla F(\bm{\hat{x}}) + \bm{\hat{v}} + \frac{1}{2\lambda}(\nabla \phi(\bm{\hat{x}})- \nabla \phi(\bm{x}))=\bm{0}.
	\end{equation}
	Since $0=D_\phi(\bm{\hat{x}},\bm{x})\geq \frac{\sigma}{2}\|\bm{\hat{x}}-\bm{x} \|^2$, we have $\bm{\hat{x}}=\bm{x}$. This together with \eqref{eq:optcond_sta} implies $F(\bm{x}) + \bm{v}=\bm{0}$ for some $\bm{v}\in\partial h(\bm{x})$.
	
	Next we prove `only if' part. As $\bm{x}$ is a stationary point, we have $\nabla F(\bm{x}) + \bm{v}=\bm{0}$ for some $\bm{v}\in\partial h(\bm{x})$. By the convexity of $h(\cdot)$, we have 
	\begin{equation}\label{eq:sta_h}
		h(\bm{\hat{x}})\geq h(\bm{x})+\langle \bm{v}, \bm{\hat{x}} - \bm{x}\rangle =h(\bm{x})+\langle -\nabla F(\bm{x}), \bm{\hat{x}} - \bm{x}\rangle.
	\end{equation}
	On the other hand, by Proposition~\ref{lem:L}, we have 
	\begin{equation}\label{eq:sta_a}
		F(\bm{\hat{x}})-F(\bm{x})-\langle \nabla F(\bm{x}), \bm{\hat{x}}-\bm{x}\rangle \geq -L D_\phi(\bm{\hat{x}},\bm{x}).
	\end{equation}
	Combine~\eqref{eq:sta_h} and~\eqref{eq:sta_a}, we have 
	\begin{equation}\label{eq:sta_b}
		F(\bm{\hat{x}})+h(\bm{\hat{x}})\geq F(\bm{x})+h(\bm{x})-L D_\phi(\bm{\hat{x}},\bm{x}).
	\end{equation}
	Note that $F(\bm{x})+h(\bm{x})\geq F(\bm{\hat{x}})+h(\bm{\hat{x}})+\frac{1}{2\lambda}D_\phi(\bm{\hat{x}},\bm{x})$ by the optimality of $\bm{\hat{x}}$, this together with~\eqref{eq:sta_b} implies $0\geq (\frac{1}{2\lambda}-L)D_\phi(\bm{\hat{x}},\bm{x})$. Since $0<\lambda<\frac{1}{2L}$, we have $\frac{1}{2\lambda}-L>0$ and hence $D_\phi(\bm{\hat{x}},\bm{x})\leq 0$. Notice $D_\phi(\bm{\hat{x}},\bm{x})\geq 0$ by its definition and hence  $D_\phi(\bm{\hat{x}},\bm{x})=0$.
	
	\section{Proof of Lemma~\ref{lem:descent}}\label{app:desc}
	We use the notation in Appendix~\ref{app:L} for simplicity. Equivalently, the updating procedure of the proposed Algorithm \ref{alg:smd} in \eqref{eq:updateA} can be re-expressed as:
	\begin{equation}\label{eq:upref}
		\begin{aligned}
			\tilde{x}_n^{t+1}&=\arg\min\ \langle \tilde{g}_n^t,y - x_n^t\rangle + h_n(y) + \frac{1}{\eta_t}D_{\phi}(y,x_n^t),\\
			\tilde{x}_n^{t+1}&=\tilde{x}_n^{t+1},\ x_i^{t+1}=x_i^{t},\ \forall i\neq n.
		\end{aligned}
	\end{equation}
	where $\tilde{g}_n^t$ is the fiber-sampled gradient with $\mathbb{E}\left[\tilde{g}_n^t \mid n \right]=\nabla_n F(\bm{x}^t)$ and $\nabla_n F(\bm{x}^t)$ denotes the partial gradient of $F(\cdot)$ with respect to block $n$ at $\bm{x}^t$. Denote $\bm{g}^t=(g_1^t,g_2^t,\ldots,g_N^t)$ with $g_i^t=0$ if $i\neq n$ and $g_n^t=\tilde{g}_n^t$, then the following proposition shows that $\bm{g}^t$ is an unbiased estimation of $\nabla F(\bm{x}^t)$ except a constant scale $\frac{1}{N}$.
	\begin{prop}\label{prop:g}
		Suppose  at iteration $t$, the selection of block index $n$ is uniformly random and the selected fiber index set $\xi$ is uniformly random conditioned on $n$, then $\mathbb{E}\left[ \bm{g}^t\right]=\frac{1}{N}\nabla F(\bm{x}^t)$.
	\end{prop}
	\begin{proof}
		By the definition of $\bm{g}^t$ and the tower property of total expectation, we have 
		\begin{equation*}
			\begin{aligned}
				&\mathbb{E}\left[ \bm{g}^t\right]\\
				&=\mathbb{E}_n\left[\mathbb{E}_{\xi} \left[ \bm{g}^t\right]\mid n\right]\\
				&=\mathbb{E}_n\left[\mathbb{E}_{\xi} \left[ (0,\ldots,\tilde{g}_n^t,\ldots,0) \mid n\right]\right]\\
				&=\mathbb{E}_n\left[(0,\ldots,\nabla_n F(\bm{x}^t),\ldots,0)\mid n\right]=\frac{1}{N}\nabla F(\bm{x}^t).
			\end{aligned}
		\end{equation*}
	\end{proof}
	Since the update in~\eqref{eq:upref} contains two types of randomness, i.e., block randomization and fiber-sampling. To clearly present the iteration analysis between iterations $t$ and $t+1$ in expectation sense, we use $\mathbb{E}_{n}[\cdot]$ to denote the expectation operation with respect to block randomization and use $\mathbb{E}_{\xi}[\cdot\mid n]$ to represent the expectation operation with respect to the fiber sampling, where $\xi$ is the randomly selected fiber indexes with block $n$ been given. 
	
	Let $\bm{x}^{t+1}$ be the iterate generated by Algorithm~\ref{alg:smd} at $\bm{x}^t$ and denote $\bm{\hat{x}}^{t+1}=\arg\min_{\bm{y}}\mathcal{L}(\bm{y};\bm{x}^{t+1})$. Then, by the definition of $\mathcal{M}(\bm{x}^{t+1})$, we have 
	\begin{equation}\label{eq:desc_M}
		\begin{aligned}
			&\mathbb{E} \left[   \mathcal{M}(\bm{x}^{t+1})   \right]=\mathbb{E} \left[   \min_{\bm{y}}\ \mathcal{L}(\bm{y};\bm{x}^{t+1})   \right]\\
			&=\mathbb{E} \left[ F(\bm{\hat{x}}^{t+1})+h(\bm{\hat{x}}^{t+1})+\frac{1}{2\lambda}D_\phi(\bm{\hat{x}}^{t+1},\bm{x}^{t+1}) \right]\\
			&\leq \mathbb{E} \left[ F(\bm{\hat{x}}^{t})+h(\bm{\hat{x}}^{t})+\frac{1}{2\lambda}D_\phi(\bm{\hat{x}}^{t},\bm{x}^{t+1}) \right],
		\end{aligned}
	\end{equation}
	where the inequality is due the fact that $\bm{\hat{x}}^{t+1}$ is the minimizer of $\mathcal{L}(\bm{y};\bm{x}^{t+1})$. By Lemma~\ref{lem:three}, $D_\phi(\bm{\hat{x}}^{t},\bm{x}^{t+1})$ can be re-epxressed as 
	\begin{equation}\label{eq:desc_D}
		\begin{aligned}
			D_\phi(\bm{\hat{x}}^{t},\bm{x}^{t+1})=&D_\phi(\bm{\hat{x}}^t,\bm{x}^t)-D_\phi(\bm{x}^{t+1},\bm{x}^t)+\\
			&\quad\langle \nabla \phi(\bm{x}^t) -\nabla \phi(\bm{x}^{t+1})  , \bm{\hat{x}}^t - \bm{x}^{t+1} \rangle.
		\end{aligned}
	\end{equation}
	Substitute~\eqref{eq:desc_D} into \eqref{eq:desc_M} and use the fact $\bm{x}^t$ and $\bm{\hat{x}}^t$ do not depend on the random procedure between iterations $t$ and $t+1$, we have 
	\begin{equation}\label{eq:desc_M2}
		\begin{aligned}
			&\mathbb{E} \left[   \mathcal{M}(\bm{x}^{t+1})   \right]\\
			&\leq F(\bm{\hat{x}}^{t})+ h(\bm{\hat{x}}^{t})+ \frac{1}{2\lambda}D_\phi(\bm{\hat{x}}^t,\bm{x}^t)+\\
			&\quad\frac{1}{2\lambda}\mathbb{E} \left[ -D_\phi(\bm{x}^{t+1},\bm{x}^t)+\langle \nabla \phi(\bm{x}^t) -\nabla \phi(\bm{x}^{t+1})  , \bm{\hat{x}}^t - \bm{x}^{t+1} \rangle ] \right]\\
			&=\mathcal{M}(\bm{x}^{t})-\frac{1}{2\lambda}\mathbb{E} \left[ D_\phi(\bm{x}^{t+1},\bm{x}^t)\right] \\
			&\quad+ \frac{1}{2\lambda}\mathbb{E} \left[ \langle \nabla \phi(\bm{x}^t) -\nabla \phi(\bm{x}^{t+1})  , \bm{\hat{x}}^t - \bm{x}^{t+1} \rangle ] \right].
		\end{aligned}
	\end{equation}
	By the tower property of expectation, the total expectation $\mathbb{E}[\cdot]$ at iteration $t$ can be decomposed as $\mathbb{E}_n\left[ \mathbb{E}_\xi [\cdot] \mid n\right]$. Then we have
	\begin{equation}\label{eq:desc_inner}
		\begin{aligned}
			&\mathbb{E}\left[\langle \nabla \phi(\bm{x}^t) -\nabla \phi(\bm{x}^{t+1})  , \bm{\hat{x}}^t - \bm{x}^{t+1}\rangle \right]\\
			&=\mathbb{E}_n\left[ \mathbb{E}_\xi \left [\langle \nabla \phi(\bm{x}^t) -\nabla \phi(\bm{x}^{t+1})  , \bm{\hat{x}}^t - \bm{x}^{t+1}\rangle \right] \mid n\right]\\
			&=\mathbb{E}_n\left[ \mathbb{E}_\xi \left[\langle \nabla \phi(x_n^t) -\nabla \phi(\tilde{x}_n^{t+1})  , \hat{x}_n^t - \tilde{x}_n^{t+1}\rangle\right] \mid n\right],
		\end{aligned}
	\end{equation}
	where the second equality is due the block randomization procedure which makes $\bm{x}^t$ and $\bm{x}^{t+1}$ have the following relation,  $$\bm{x}^{t+1}=[x_1^t,\ldots,x_{n-1}^t,\tilde{x}_n^{t+1},x_{n+1}^t,\ldots,x_N^t].$$
	Next, by the first-order optimality condition (Lemma~\ref{lem:optcond}) for the subproblem in~\eqref{eq:upref}, we have 
	\begin{equation}\label{eq:desc_optcond}
		\begin{aligned}
			&\langle \nabla \phi(x_n^t) -\nabla \phi(\tilde{x}_n^{t+1})  , \hat{x}_n^t - \tilde{x}_n^{t+1}\rangle\\
			&\leq \langle \eta_t (\tilde{g}_n^t+\tilde{v}_n^{t+1})  , \hat{x}_n^t - \tilde{x}_n^{t+1} \rangle,
		\end{aligned}
	\end{equation}
	where $\tilde{v}_n^{t+1}\in\partial h_n(\tilde{x}_n^{t+1})$. Since $h_n(\cdot)$ is convex, we have 
	\begin{equation}\label{eq:desc_cvx}
		\begin{aligned}
			\langle \tilde{v}_n^{t+1}, \hat{x}_n^t - \tilde{x}_n^{t+1} \rangle\leq h_n(\hat{x}_n^{t})-h_n(\tilde{x}_n^{t+1})=0,
		\end{aligned}
	\end{equation}
	where the last equality is due the definition of $h_n(\cdot)$ and $\hat{x}_n^t,\tilde{x}_n^{t+1}\in\mathcal{X}_n$. Substitute~\eqref{eq:desc_optcond} and~\eqref{eq:desc_cvx} into~\eqref{eq:desc_inner}, we have 
	\begin{equation}\label{eq:desc_g}
		\begin{aligned}
			&\mathbb{E}\left[\langle \nabla \phi(\bm{x}^t) -\nabla \phi(\bm{x}^{t+1})  , \bm{\hat{x}}^t - \bm{x}^{t+1}\rangle \right]\\
			&\leq \eta_t \mathbb{E}_n\left[ \mathbb{E}_\xi \left[\langle  \tilde{g}_n^t  , \hat{x}_n^t - \tilde{x}_n^{t+1}\rangle \right] \mid n\right]\\
			&= \eta_t \mathbb{E}_n\left[ \mathbb{E}_\xi \left[\langle  \tilde{g}_n^t  , \hat{x}_n^t -x_n^t + x_n^t- \tilde{x}_n^{t+1}\rangle\right] \mid n\right]\\
			&= \eta_t \mathbb{E}_n\left[ \mathbb{E}_\xi \left[\langle  \tilde{g}_n^t  , \hat{x}_n^t -x_n^t\rangle \right] \mid n\right]+ \mathbb{E}_n\left[ \mathbb{E}_\xi \left[\langle  \eta_t\tilde{g}_n^t, x_n^t- \tilde{x}_n^{t+1}\rangle\right] \mid n\right]\\
			&\mathop{=}\limits^{(i)}\eta_t \mathbb{E}_n\left[\langle \nabla_n F(\bm{x}^t)  , \hat{x}_n^t -x_n^t \rangle \right]+\mathbb{E}_n\left[ \mathbb{E}_\xi \left[\langle  \eta_t\tilde{g}_n^t, x_n^t- \tilde{x}_n^{t+1}\rangle \right] \mid n\right]\\
			&=\frac{\eta_t}{N}\langle \nabla F(\bm{x}^t)  , \bm{\hat{x}}^t -\bm{x}^t \rangle +\mathbb{E}_n\left[ \mathbb{E}_\xi \left[\langle  \eta_t\tilde{g}_n^t, x_n^t- \tilde{x}_n^{t+1}\rangle \right] \mid n\right],
		\end{aligned}
	\end{equation}
	where equality $(i)$ is because fiber sampling gives unbiased gradient estimation, $\mathbb{E}_{\xi}\left[\tilde{g}_n^t\mid n\right]=\nabla_n F(\bm{x}^t)$. Next, we bound the term $\mathbb{E}_n\left[ \mathbb{E}_\xi \left[\langle  \eta_t\tilde{g}_n^t, x_n^t- \tilde{x}_n^{t+1}\rangle \right] \mid n\right]$.
	\begin{equation}\label{eq:desc_eta}
		\begin{aligned}
			&\mathbb{E}_n\left[ \mathbb{E}_\xi \left[\langle  \eta_t\tilde{g}_n^t, x_n^t- \tilde{x}_n^{t+1}\rangle \right] \mid n\right]\\
			&=\mathbb{E}_n\left[ \mathbb{E}_\xi \left[\langle  \frac{\eta_t}{\sqrt{\sigma}}\tilde{g}_n^t, \sqrt{\sigma}(x_n^t- \tilde{x}_n^{t+1})\rangle \right] \mid n\right]\\
			&\mathop{\leq}\limits^{(i)}\mathbb{E}_n\left[ \mathbb{E}_\xi \left[\frac{\eta_t^2}{2\sigma} \| \tilde{g}_n^t\|^2  + \frac{\sigma}{2} \| x_n^t - \tilde{x}_n^{t+1}\|^2\right] \mid n\right]\\
			&\mathop{=}\limits^{(ii)}\frac{\eta_t^2}{2\sigma}\mathbb{E}\left[\| \bm{g}^t \|^2\right]+\frac{\sigma}{2}\mathbb{E}\left[\| \bm{x}^t - \bm{x}^{t+1} \|^2\right]\\
			&\mathop{\leq}\limits^{(iii)}\frac{\eta_t^2}{2\sigma}\mathbb{E}\left[\| \bm{g}^t \|^2\right]+\mathbb{E}\left[D_{\phi}(\bm{x}^t, \bm{x}^{t+1})\right],
		\end{aligned}
	\end{equation}
	where $(i)$ is due fact $\| x\|^2+\| y\|^2\geq 2\langle x, y\rangle$, $(ii)$ is by the definitions of $\bm{g}^t$ and $\bm{x}^{t+1}$ that only block $n$ been changed, and $(iii)$ is because the $\sigma$-strongly convexity of $\phi(\cdot)$, i.e., $D_{\phi}(\bm{x}^t, \bm{x}^{t+1})\geq \frac{\sigma}{2}\| \bm{x}^t - \bm{x}^{t+1} \|^2$. Substitute~\eqref{eq:desc_eta} into~\eqref{eq:desc_g}, which together with~\eqref{eq:desc_M2} implies
	\begin{equation}\label{eq:desc_M3}
		\begin{aligned}
			&\mathbb{E} \left[   \mathcal{M}(\bm{x}^{t+1})   \right]\\
			&=\mathcal{M}(\bm{x}^{t}) + \frac{\eta_t}{2N\lambda}\langle \nabla F(\bm{x}^t)  , \bm{\hat{x}}^t -\bm{x}^t \rangle+ \frac{\eta_t^2}{4\sigma\lambda}\mathbb{E}\left[\| \bm{g}^t \|^2\right]\\
			&\mathop{\leq}\limits^{(i)}\mathcal{M}(\bm{x}^{t}) - \frac{\eta_t}{2N\lambda}(\frac{1}{2\lambda}-L)D_\phi(\bm{\hat{x}}^t,\bm{x}^t) + \frac{\eta_t^2}{4\sigma\lambda}\mathbb{E}\left[\| \bm{g}^t \|^2\right]\\
			&=\mathcal{M}(\bm{x}^{t}) - \eta_t\frac{1-2\lambda L}{4\lambda^2N}D_\phi(\bm{\hat{x}}^t,\bm{x}^t) + \frac{\eta_t^2}{4\sigma\lambda}\mathbb{E}\left[\| \bm{g}^t \|^2\right],
		\end{aligned}
	\end{equation}
	where $(i)$ is by Lemma~\ref{lem:nocvx_bound}. This completes the proof of Lemma~\ref{lem:descent}.
	
	\begin{lem}[Three Point Equality]\label{lem:three}
		For any $x,y,z\in \textbf{dom}\ \phi$, we have 
		$$D_\phi(x,z)-D_\phi(x,y)-D_\phi(y,z)=\langle \nabla \phi(y) - \nabla \phi(z), x - y \rangle.$$
	\end{lem}
	
	\begin{lem}[Optimality Condition for~\eqref{eq:upref}]\label{lem:optcond}
		Let $v^{t+1}\in\partial h_n(x^{t+1}_n)$, then 
		\begin{equation}\label{eq:optcond}
			\begin{aligned}
				\langle \hat{g}_n^t +v_n^{t+1}+ \frac{1}{\eta_t}(\nabla \phi(x^{t+1}_n) - \nabla \phi(x^{t}_n)), y - x^{t+1}_n\rangle\geq 0,\ \forall y\in\mathcal{X}_n,
			\end{aligned}
		\end{equation}
	\end{lem}
	
	\begin{lem}[Optimality of $\bm{\hat{x}}$]
		\label{lem:nocvx_bound}
		Let $\bm{\hat{x}}=\arg\min_{\bm{y}} \mathcal{L}(\bm{y};\bm{x})$ with $\bm{x}\in\mathcal{X}$. Then, we have
		\begin{equation}\label{eq:optcondevp2}
			\begin{aligned}
				-(\frac{1}{2\lambda}-L)D_\phi(\bm{\hat{x}},\bm{x})\geq \langle \nabla F(\bm{x}), \bm{\hat{x}} - \bm{x} \rangle
			\end{aligned}
		\end{equation}
	\end{lem}
	\begin{proof}
		By Lemma~\ref{lem:L}, we have 
		$$F(\bm{\hat{x}}) - F(\bm{x}) -\langle \nabla F(\bm{x}), \bm{\hat{x}}-\bm{x}\rangle\geq -L D_\phi(\bm{\hat{x}},\bm{x}).$$ On the other hand, by the definition of $\bm{\hat{x}}$, we have $$F(\bm{\hat{x}})+h(\bm{\hat{x}})+\frac{1}{2\lambda}D_\phi(\bm{\hat{x}},\bm{x})\leq F(\bm{x})+h(\bm{x}).$$ 
		Sum up the above two inequalities completes the proof since $h(\bm{\hat{x}})=h(\bm{x})=0$.
	\end{proof}
	
	\section{Proof of Theorem~\ref{thm}}\label{app:thm}
	Recall Lemma~\ref{lem:descent}, as the optimization process continues in a Markovian manner~\cite{bottou2018optimization}, take the total expectation and sum up~\eqref{eq:desc_lem} in Lemma~\ref{lem:descent} from $t=0$ to $t=T$, we have 
	\begin{equation}
		\begin{aligned}
			\mathbb{E}\left[ \mathcal{M}(\bm{x}^{T}) \right]&\leq \mathcal{M}(\bm{x}^0)-\frac{(1-2\lambda L)}{4\lambda^2N}\sum_{t=0}^{T-1}\eta_t D_\phi(\bm{\hat{x}}^t,\bm{x}^t)\\
			&\quad+\frac{1}{4\lambda\sigma}\sum_{t=0}^{T-1}\eta_t^2\mathbb{E}\left[\|\bm{g}^t\|^2 \right].
		\end{aligned}
	\end{equation}
	Let $\bm{x}^*$ be the global minimizer Problem~\eqref{eq:optref}, we have $F(\bm{x}^*)\leq \mathbb{E}\left[ \mathcal{M}(\bm{x}^{T})\right]$ by the non-negativity of $D_{\phi}(\cdot,\cdot)$. This together with the above inequality implies 
	\begin{equation}\label{eq:sumupT}
		\begin{aligned}
			&\frac{(1-2\lambda L)}{4\lambda^2N}\sum_{t=0}^{T-1}\eta_t D_\phi(\bm{\hat{x}}^t,\bm{x}^t) \\
			&\leq \mathcal{M}(\bm{x}^0)-F(\bm{x}^*)
			+\frac{1}{4\lambda\sigma}\sum_{t=0}^{T-1}\eta_t^2\mathbb{E}\left[\|\bm{g}^t\|^2 \right].
		\end{aligned}
	\end{equation}
	If the step size sequence $\{\eta_t\}$ is diminishing, i.e.,  $\sum_{t=0}\eta_t=\infty$ and $\sum_{t=0}\eta_t^2<\infty$. Then we have $\liminf_{t\rightarrow \infty} \mathbb{E}\left[ D_\phi(\bm{\hat{x}}^t,\bm{x}^t) \right]=0$. 
	
	For a constant step size $\eta_t=\frac{1}{\sqrt{T}}$, substitute it into~\eqref{eq:sumupT} and  divide $\sqrt{T}$ on both side, we have 
	$$\frac{1}{T}\sum_{t=0}^{T-1}\mathbb{E}\left[D_\phi(\bm{\hat{x}}^t,\bm{x}^t) \right]\leq \frac{C}{\sqrt{T}},$$
	where $C=\frac{4\lambda^2N}{1-2\lambda L}\left[\mathcal{M}(\bm{x}^0)-F(\bm{x}^*) + \frac{1}{4\lambda\sigma}G \right]$ and $G$ is the upper bound of $\mathbb{E}\left[ \| \bm{g}^t\|^2\right]$. This immediately implies $\min_{0\leq t \leq T-1}\mathbb{E}\left[D_\phi(\bm{\hat{x}}^t,\bm{x}^t)\right] \leq \frac{C}{\sqrt{T}}$. 
	
	\section{Tables}
	\begin{table}[h]
		\renewcommand\tabcolsep{2.0pt}
		\centering
		\scriptsize 
		\caption{Explicit forms of $\hat{\bm{G}}^t_n$ for different $\ell(\cdot)$.}
		\begin{tabular}{c c}
			\toprule
			\textbf{Loss Function $\ell(\cdot)$} & $|\mathcal{F}_n|I_n\cdot\hat{\bm{G}}^t_n$\\
			\toprule
			$\frac{1}{2}(x-m)^2$  &$\left(\hat{\bm{H}}_n \bm{A}_n^t-{\hat{\bm{X}}}_n\right)^T\hat{\bm{H}}_n$ \\
			$\frac{x}{m+\epsilon}+\log (m+\epsilon)$ &$\left[\left(\hat{\bm{H}}_n \bm{A}_n^t-{\hat{\bm{X}}}_n+\epsilon\right)\circledast\left(\hat{\bm{H}}_n \bm{A}_n^t+\epsilon\right)^{.-2}\right]^T\hat{\bm{H}}_n $  \\
			$m-x\log(m+\epsilon)$ & $\left[-{\hat{\bm{X}}}_n\circledast\left(\hat{\bm{H}}_n \bm{A}_n^t+\epsilon\right)^{.-1}+1\right]^T\hat{\bm{H}}_n$ \\
			$e^m-xm$ & $\left[\exp(\hat{\bm{H}}_n \bm{A}_n^t)-{\hat{\bm{X}}}_n\right]^T\hat{\bm{H}}_n$ \\
			$\log(m+1)-x\log(m+\epsilon)$  & $\left(\hat{\bm{H}}_n \bm{A}_n^t+1\right)^{.-1}-{\hat{\bm{X}}}_n\circledast\left(\hat{\bm{H}}_n \bm{A}_n^t+\epsilon\right)^{.-1}$ \\
			$\log(1+e^m)-xm$& $\left[\exp(\hat{\bm{H}}_n \bm{A}_n^t)\circledast\left(\exp(\hat{\bm{H}}_n \bm{A}_n^t)+1\right)^{.-1}+{\hat{\bm{X}}}_n\right]^T\hat{\bm{H}}_n$ \\
			$\beta$-divergence & $\left[ ({\hat{\bm{H}}}_n \bm{A}_n^t+\epsilon)^{.\beta-2}\circledast (\hat{\bm{H}}_n \bm{A}_n^t-{\hat{\bm{X}}}_n+\epsilon)\right]^T{\hat{\bm{H}}}_n$\\
			\bottomrule
		\end{tabular}
		\label{tab:G}
	\end{table}
	
	\newpage
	\begin{table}[h]
		\renewcommand\tabcolsep{2.0pt}
		\centering
		\small
		\caption{$(\phi,\bm{\Gamma}_n^t)$ based on Jensen's inequality.}
		\begin{threeparttable}
			\begin{tabular}{c c c c}
				\toprule
				\textbf{Loss Function} & $\phi(a)$ & $|\mathcal{F}_n|I_n\cdot\bm{\Gamma}_n^t$\\
				\toprule
				Eucl. Dis.  & $\frac{1}{2}a^2$  &$\frac{1}{2}(\hat{\bm{X}}_n^t)^T\hat{\bm{H}}_n\oslash {\bm{A}}_n^t $ \\
				
				IS Div. &$\frac{1}{a}$ &${\bm{A}}_n^t\circledast \left[{\hat{\bm{X}}}_n\oslash(\hat{\bm{X}}_n^t)^{.2}\right]^T\hat{\bm{H}}_n $  \\
				
				KL Div. & $-\log a$ & $(\bm{A}_n^t)^T\circledast\left[\hat{\bm{H}}_n^T({\hat{\bm{X}}}_n\oslash(\hat{\bm{X}}_n^t))\right]$ \\
				
				$\beta$-Div. ($\beta>1$) & $a^{\beta}$ & $\frac{1}{\beta}(\hat{\bm{X}}_n^t)^{.\beta-1})^T\hat{\bm{H}}_n\oslash (\bm{A}_n^t)^{.\beta-1}$\\
				
				$\beta$-Div. ($\beta<1$) & $a^{\beta-1}$ & $\frac{1}{(1-\beta)}[{\hat{\bm{X}}}_n\circledast(\hat{\bm{X}}_n^t)^{.\beta-2})^T]\hat{\bm{H}}_n\oslash (\bm{A}_n^t)^{.\beta-2}$\\
				\bottomrule
			\end{tabular}
			
			\begin{tablenotes}
				\footnotesize
				\item {\scriptsize *In this table, $\hat{\bm{X}}_n^t=\hat{\bm{H}}_n \bm{A}_n^t$.}
			\end{tablenotes}
		\end{threeparttable}
		\label{tab:jensen}
	\end{table}
	\clearpage
	
	\onecolumn 
	\section{Figures}
	
	\begin{figure}[h]
		\centering
		\subfigure[$\approx5\%$ nonzero entries]{
			\includegraphics[width=0.23\linewidth]{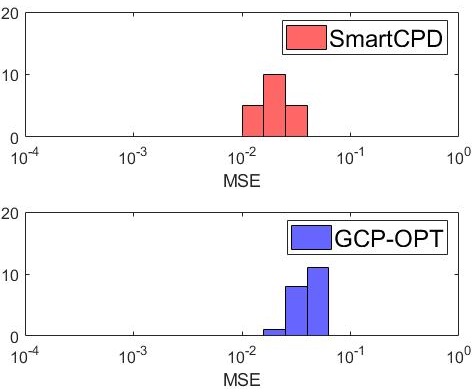}}
		\subfigure[$\approx15\%$ nonzero entries]{
			\includegraphics[width=0.23\linewidth]{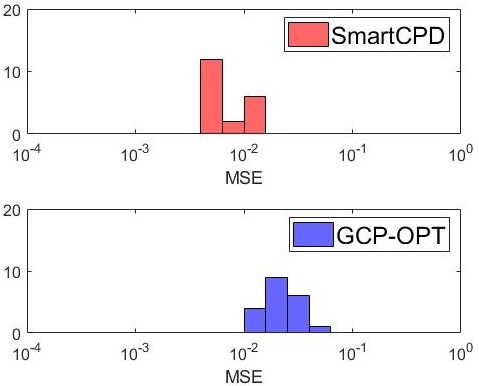}}
		\subfigure[$\approx28\%$ nonzero entries]{
			\includegraphics[width=0.23\linewidth]{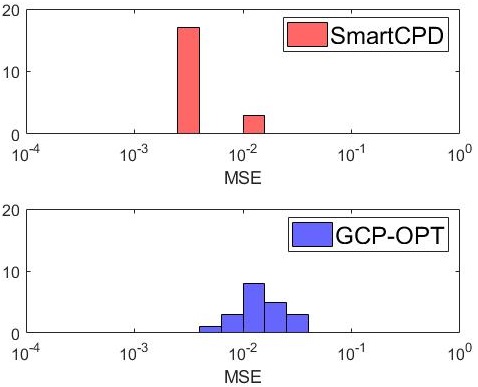}}
		\subfigure[$\approx40\%$ nonzero entries]{
			\includegraphics[width=0.23\linewidth]{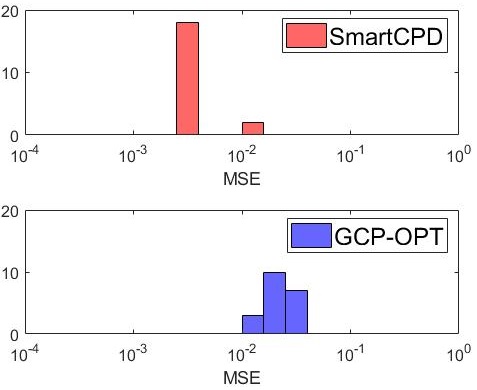}}
		\caption{Histogram of MSE after 60 seconds of $100\times100\times100$ binary tensor (rank $20$) with different level of sparsity.}
		\label{fig:binary_hist}
	\end{figure}
	\begin{figure}[h]
		\centering
		\includegraphics[width=1\linewidth]{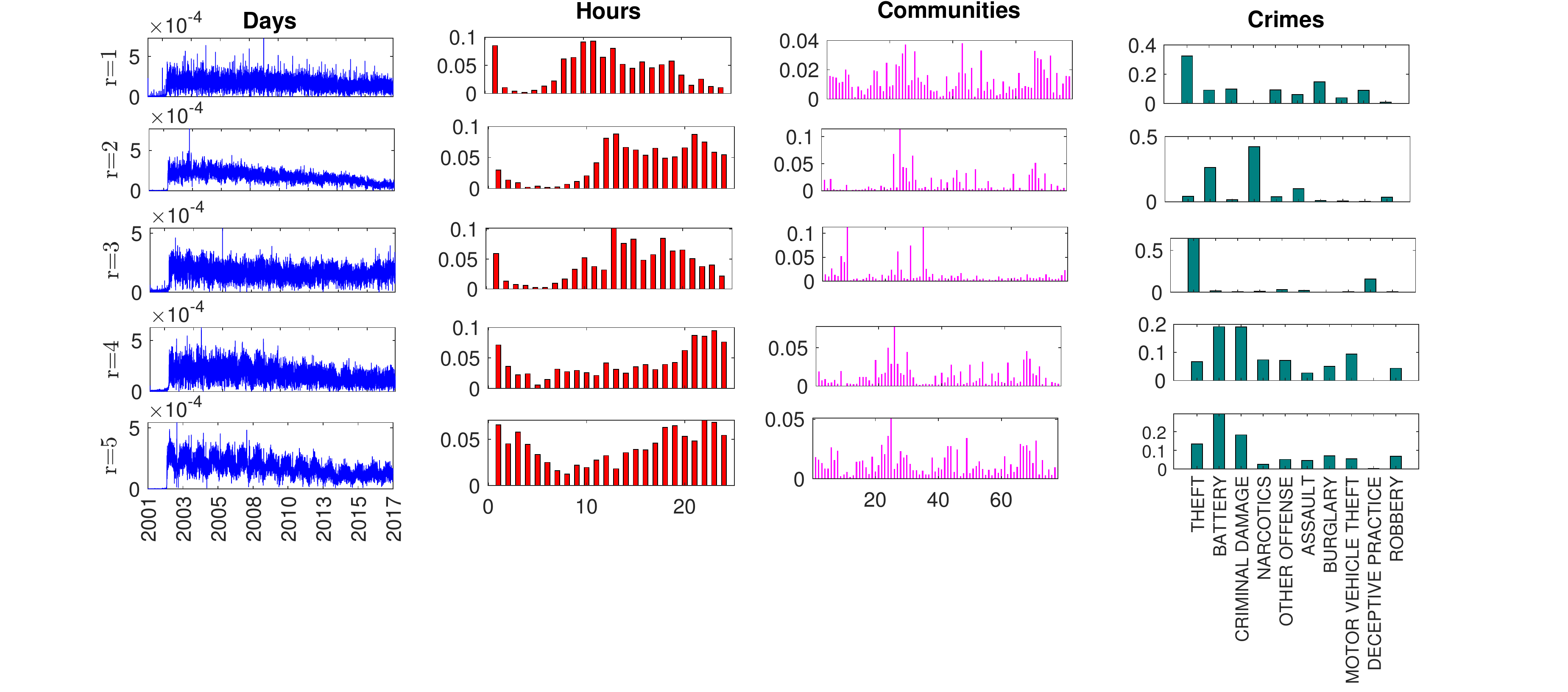}
		\caption{{Learned latent factors by \texttt{SmartCPD} when it reached the stopping criterion with cost value 0.043 (time= 35.81 sec.). }}
		\label{fig:factor_chicago1}
	\end{figure}

\end{document}